\documentclass[a4paper]{article}
\usepackage{natbib}

\usepackage{url}
\usepackage[hidelinks]{hyperref}
\usepackage[utf8]{inputenc}
\usepackage[small]{caption}
\usepackage{graphicx}
\usepackage{amsmath}
\usepackage{amsfonts}
\usepackage{amsthm,amssymb}
\usepackage{booktabs}
\usepackage{algorithm}
\usepackage{algorithmic}

\title{Rediscovering Argumentation Principles Utilizing Collective Attacks}

\author{%
	Wolfgang Dvo\v{r}\'ak$^1$\and
	Matthias K\"{o}nig$^1$\and
	Markus Ulbricht$^2$\and
	Stefan Woltran$^1$\\
	$^1$TU Wien, Institute of Logic and Computation\\
	$^2$Leipzig University, Department of Computer Science\\
	\{dvorak, mkoenig, woltran\}@dbai.tuwien.ac.at\\
	mulbricht@informatik.uni-leipzig.de
}
\date{}

\usepackage{todonotes}

\usepackage{pifont}

\newcommand{\ie}{i.e.\ }
\newcommand{\TODO}[1]{}

\newcommand{\cm}{\ding{51}}
\newcommand{\xm}{\ding{55}}

\usepackage{complexity}

\usepackage{tikz}
\usetikzlibrary{shapes,decorations,shadows,arrows,calc}
\makeatletter
\tikzstyle{arg}=[draw,circle,fill=gray!15,inner sep=1pt,minimum size=.5cm]
\tikzstyle{argd}=[draw,circle,gray!50,inner sep=1pt,minimum size=.5cm]

\tikzstyle{argTD}=[draw, thick, circle, fill=gray!15,inner sep=0pt,minimum size=0.6cm,font=\small]
\tikzstyle{argR}=[draw, thick, circle, fill=gray!15,inner sep=0pt,minimum size=0.45cm,font=\small]
\tikzstyle{scc}=[draw, thick, rectangle,align=center, fill=gray!15,inner sep=0pt,minimum size=0.8cm,font=\small, rounded corners=0ex,]
\tikzstyle{argTDX}=[draw, dotted,thick, circle, inner sep=0pt,minimum size=0.6cm,font=\small]
\tikzstyle{sccX}=[draw,dotted, thick, rectangle,align=center, inner sep=0pt,minimum size=0.8cm,font=\small, rounded corners=0ex,]
\tikzstyle{argsmall}=[draw, thick, circle, fill=gray!15,inner sep=0pt,minimum size=0.4cm]
\tikzstyle{argsmallX}=[draw, thick, circle, inner sep=0pt,minimum size=0.3cm,dotted]
\tikzstyle{bag}=[draw,rectangle, rounded corners=1ex,
align=center, inner sep=1ex]

\tikzstyle{atts}=[draw,thick, inner sep=5pt, rounded corners=3pt]

\newcommand{\att}{\ensuremath{\at}}
\newcommand{\F}{SF}

\newcommand{\at}{\mapsto}

\newcommand{\cf}{\textit{cf}}
\newcommand{\adm}{\textit{adm}}
\newcommand{\com}{\textit{com}}
\newcommand{\comp}{\textit{com}}
\newcommand{\naive}{\textit{naive}}
\newcommand{\stb}{\textit{stb}}
\newcommand{\stable}{\stb}
\newcommand{\prf}{\textit{pref}}
\newcommand{\pref}{\textit{pref}}
\newcommand{\stage}{\textit{stage}}

\newcommand{\sem}{\textit{sem}}
\newcommand{\semi}{\textit{sem}}
\newcommand{\grd}{\textit{grd}}

\newcommand{\Aa}{\mathfrak{A}}

\newtheorem{theorem}{Theorem}[section]
\newtheorem{lemma}[theorem]{Lemma}
\newtheorem{proposition}[theorem]{Proposition}

\newtheorem{observation}[theorem]{Observation}
\theoremstyle{definition}
\newtheorem{definition}[theorem]{Definition}
\newtheorem{example}[theorem]{Example}

\newcommand{\proj}[2]{{#1}{\downarrow}_{#2}}
\newcommand{\projI}[2]{\proj{#1}{#2}}

\newcommand{\projIII}[3]{{#1}{\Downarrow}^{#2}_{#3}}
\newcommand{\pte}[1]{\projIII{SF}{(E\setminus #1)^+}{\UPscc{SF}{#1}{E}}}

\newcommand{\Uscc}[3]{\mathit{U}_{#1}({#2,#3})}
\newcommand{\Dscc}[3]{\mathit{D}_{#1}({#2,#3})}
\newcommand{\Pscc}[3]{\mathit{P}_{#1}({#2,#3})}
\newcommand{\UPscc}[3]{\mathit{UP}_{#1}({#2,#3})}
\newcommand{\Mscc}[3]{\mathit{M}_{#1}({#2,#3})}

\newcommand{\primal}{\mathit{primal}}
\newcommand{\SCCs}{\mathit{SCCs}}

\newcommand{\GF}{\mathcal{GF}}
\newcommand{\BF}{\mathcal{BF}}

\newcommand{\USets}{\mathit{US}}

\begin{document}
	
	\maketitle
	
	\begin{abstract}
		Argumentation Frameworks (AFs) are a key formalism in AI research. 
		Their semantics have been investigated in terms of \emph{principles}, 
		which define 
		characteristic properties 
		in order
		to deliver guidance for 	
		analysing established and developing new semantics.
		Because of the simple structure of AFs, many desired properties hold almost trivially,  at the same time hiding interesting concepts behind syntactic notions.
		We extend the principle-based approach to Argumentation Frameworks with Collective Attacks (SETAFs) and provide a comprehensive overview of common principles for their semantics.
		Our analysis shows that investigating principles based on decomposing the given SETAF (e.g.\ directionality or SCC-recursiveness) 
		poses 
		additional challenges in comparison to usual AFs.   
		We introduce the notion of the reduct as well as the modularization principle for SETAFs which will prove beneficial for this kind of investigation.
		We then demonstrate how our findings can be utilized for incremental computation of extensions and 
		give a novel
		parameterized tractability result for verifying preferred extensions. 
	\end{abstract}
	
	\section{Introduction}
	Abstract argumentation frameworks (AFs) as proposed by Dung (\citeyear{Dung95}) in his seminal paper are nowadays a classical research area in knowledge representation and reasoning. 
	In AFs, arguments are interpreted as abstract entities and thus the focus is solely on the relationship between them, i.e.\ which arguments are in conflict with each other. 
	Consequently, an AF is simply a directed graph, where the vertices are interpreted as arguments and edges as attacks between them. 
	Within the last years, several semantics for AFs have been proposed in order to formalize jointly acceptable sets of arguments, see, e.g., \citep{BaroniCG18}. 
	Different semantics have different features, yielding more or less beneficial behavior in varying contexts. 
	In order to assess and compare the characteristics of semantics in a formal and objective way, researchers pay increasing attention to perform principles-based analyses of AF semantics, \ie formalizing properties semantics should 
	satisfy in different situations.  
	We refer the reader to \citep{TorreV17} for a recent overview. 
	
	In the present paper we consider Argumentation Frameworks with collective attacks (SETAFs), introduced by \citeauthor{NielsenP06}
	~(\citeyear{NielsenP06}). 
	SETAFs generalize Dung-style AFs in the sense that some arguments can only be effectively defeated by a collection of attackers, yielding a natural representation as a directed hypergraph.
	Many key semantic properties of AFs have been shown to carry over to SETAFs, see e.g.~\citep{NielsenP06,FlourisB19}.
	Moreover, work has been done on expressiveness~\citep{DvorakFW19},
	and translations from SETAFs to AFs~\citep{Polberg17,FlourisB19}.
	Also the hypergraph structure of SETAFs has recently been subject of investigation~\citep{DvorakKW21,DvorakKW21b}. 
	However, a thorough principle-based analysis of SETAF semantics is still unavailable. 
	In this paper, we will close this gap by investigating the common SETAF semantics w.r.t.\ to a comprehensive selection of principles, inspired from similar proposals that have been studied for Dung-AFs. 
	
	Although we will see that in many cases the behavior generalizes from AFs to the setting with collective attacks, our study also reveals situations where caution is required and thus emphasizes properties we deem natural for AFs. In fact, many AF principles like SCC-recursiveness \citep{BaroniGG05a} or the recently introduced modularization property \citep{baumann2020comparing} are concerned with partial evaluation of the given graph and step-wise computation of extensions. 
	We will pay special attention to these kind of principles since (a) they require to establish novel technical foundations when generalizing the underlying structure from simple graphs to hypergraphs and (b) have immediate implications for the design of solvers.
	Along the way, we will also introduce a SETAF version for the reduct \citep{BaumannBU2020} of an AF which has proven to be a handy tool when investigating argumentation semantics. 
	
	The main contribution of this paper is to show that our natural extensions of the AF principles are well-behaving for SETAFs. 
	We show that basic properties are preserved, as well as their implications in terms of the structure of extensions. More specifically our contributions are  as follows.
	\begin{itemize}
		\item After giving necessary preliminaries in Section~\ref{section:preliminaries} we introduce for SETAFs the  counterparts to the most important basic principles from the literature in Section~\ref{section:setaf principles}.
		\item We propose the $E$-reduct $SF^E$ for a SETAF $SF$ and a set $E$ of arguments and investigate its core properties, including the modularization property for SETAFs (Section~\ref{section:reduct}). Moreover, we use the reduct to provide alternative characterizations of SETAFs semantics.
		\item We introduce uninfluenced sets of arguments in SETAFs as the counterpart of unattacked sets in AFs. We then propose and investigate a SETAF version of the directionality and non-interference principles (Section~\ref{sec:dir}) and the SCC-recursiveness (Section~\ref{section:sccs}).
		\item We discuss the computational implications of modularization, directionality and
		SCC-recursiveness in Section~\ref{sec:incremental}. In particular we illustrate the potential for incremental algorithms.
	\end{itemize}
	Technical proofs that are omitted due to space constraints are provided in the appendix.
	Some of the results reported in this article evolved from the paper \citep{DvorakKUW21} presented  at the NMR 2021 workshop.
	
	\section{Background}
	\label{section:preliminaries}
	
	We briefly recall the definitions of SETAFs and its semantics (see, e.g.,~\citep{BikakisCSFP21}).
	Throughout the paper, we assume a countably infinite domain $\Aa$ of possible arguments. 
	
	A SETAF is a pair $SF=(A,R)$ where $A\subseteq \Aa$ is 
	a finite set of \emph{arguments},
	and $R \subseteq (2^A \setminus \{\emptyset\}) \times A$ is the \emph{attack relation}. For an attack $(T,h)\in R$ we call $T$ the \emph{tail} and $h$ the \emph{head} of the attack.
	SETAFs $(A,R)$, where for all $(T,h) \in R$ it holds that $|T|=1$, amount to 
	(standard Dung) AFs. In that case, we usually write $(t,h)$ to denote the set-attack $(\{t\},h)$.
	Moreover, for a SETAF $SF=(A,R)$, we use $A(SF)$ and $R(SF)$ to identify
	its arguments $A$ and its attack relation $R$, respectively.
	
	Given a SETAF $(A,R)$, 
	we write $S \at_R a$ 
	if there is a set $T \subseteq S$
	with $(T,a) \in R$.
	Moreover, we write $S' \at_R S$ if 
	$S' \at_R a$ for some $a \in S$.
	For $S \subseteq A$, we use $S_R^+$ to denote the set $\{a \mid S \at_R a\}$ 
	and define the \emph{range} of $S$ (w.r.t.\ $R$), denoted $S_R^\oplus$, as the set 
	$S \cup S_R^+$. 
	
	\begin{example}
		\label{ex:setaf}
		Consider the SETAF $SF=(A,R)$ with arguments
		
		$A= \{a,b,c, d,e,f, g,h\}$ and attack relation 
		\begin{align*}
		R=\{ &(a,b),(\{b,d\},c),(b,d),(d,b), (d,e),(e,d),\\ 
		&(\{d,f\},h), (f,g),(g,f), (g,h),(h,g)\};
		\end{align*}
		the collective attacks $(\{b,d\},c),(\{d,f\},h)$ are highlighted.
		
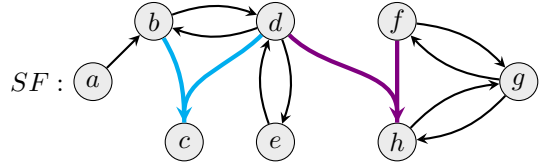
\begin{center}
			\begin{tikzpicture}[scale=0.8,>=stealth]
			\path 
			(-1,-1)node[arg,label={left:$SF:$}] (a){$a$}
			(0,0)node[arg] (b){$b$}
			(0.5,-2)node[arg] (c){$c$}
			(2,0)node[arg] (d){$d$}
			(2,-2)node[arg] (e){$e$}
			(4,0)node[arg] (f){$f$}
			(4,-2)node[arg] (h){$h$}
			(6,-1)node[arg] (g){$g$}
			;
			\path[thick,->]
			(a)edge(b)
			(b)edge[ultra thick,color=cyan,out=-60,in=90,-](c)
			(d)edge[ultra thick,color=cyan,out=225,in=90](c)
			(d)edge[ultra thick,color=violet,out=-45,in=90](h)
			(f)edge[ultra thick,color=violet,out=-90,in=90,-](h)
			;
			\path[thick,->,bend left=20]
			(d)edge(b)
			(b)edge(d)
			(d)edge(e)
			(e)edge(d)
			(f)edge(g)
			(g)edge(f)
			(h)edge(g)
			(g)edge(h)
			;
			\end{tikzpicture}
		\end{center}
	\end{example}
	
	The well-known notions of conflict and defense from classical Dung-style AFs naturally generalize to SETAFs.
	
	\begin{definition}
		Given a SETAF $SF=(A,R)$, a set $S\subseteq A$ is \emph{conflicting} in $SF$ 
		if $S \at_R a$ for some $a\in S$.
		A set
		$S \subseteq A$ is \emph{conflict-free} in $SF$, if $S$ is not conflicting in $SF$, i.e.\
		if $T \cup \{h\} \not\subseteq S$ for each $(T,h) \in R$.
		$\cf (SF)$ denotes
		the set of all conflict-free sets in $SF$.
	\end{definition}
	
	\begin{definition}
		Given a SETAF $SF=(A,R)$, 
		an argument $a\! \in\! A$ is \emph{defended} (in $SF$) by a set $S \subseteq A$ if for each $B \subseteq A$, 
		such that $B \at_R a$, also $S \at_R B$.
		A set $T \subseteq A$  is defended (in $SF$) by $S$ if each $a\in T$ is 
		defended by $S$ (in $SF$).
	\end{definition}
	Moreover, we make use of the \emph{characteristic function} $\Gamma_{SF}$ of a SETAF $SF=(A,R)$,
	defined as $\Gamma_{SF}(S)= \allowbreak \{a\in A \mid S \text{ defends }a\text{ in }SF\}$ for $S\subseteq A$.
	
	The semantics we study in this work are
	the
	grounded,
	admissible, complete, 
	preferred, stable, 
	naive,
	stage and semi-stable 
	semantics,
	which we will abbreviate by
	$\grd$, $\adm$, $\comp$, $\pref$, $\stable$, 
	$\naive$,
	$\stage$ and $\semi$,
	respectively~\citep{BikakisCSFP21}. 
	\begin{definition}
		Given a SETAF $SF=(A,R)$ 
		and a conflict-free set $S\in\cf(SF)$.
		Then,
		\vspace{-1pt}
		\begin{itemize}
			
			\item $S\in \adm(SF)$, if $S$ defends itself in $SF$,	
			\item $S \in \comp(SF)$, if $S \in \adm(SF)$ and $a \in S$ for all $a \in A$
			defended by $S$,	
			\item $S \in \grd(SF)$, if $S = \bigcap_{T\in\comp(SF)}{T}$,
			\item $S \in \prf(SF)$, if $S \in \adm(SF)$ and $\nexists  T \in \adm(SF)$
			s.t.\ $T \supset S$,
			\item $S \in \stb(SF)$, if $S\at_R a$ for all $a\in A\setminus S$,
			\item $S \in \naive(SF)$, if $\nexists T \in \cf(SF)$ with $T \supset S$,
			\item $S \in \stage(SF)$, if $\nexists T \in \cf(SF)$ with $T_R^\oplus \supset S_R^\oplus$, and
			\item $S \in \sem(SF)$, if $S \in \adm(SF)$ and 
			$\nexists T \in \adm(SF)$
			s.t.\ 
			$T_R^\oplus \supset S_R^\oplus$.
		\end{itemize}
	\end{definition}
	The relationship between the semantics has been clarified
	in~\citep{DvorakGW18,FlourisB19,NielsenP06} and 
	matches with the relations between the semantics for Dung AFs, i.e.\
	for any SETAF $SF$:
	\begin{gather*}
	\stb(SF)\subseteq \sem(SF)\subseteq \pref(SF) \subseteq \comp(SF) \subseteq \adm(SF)\\
	\stb(SF) \subseteq\ \stage(SF)\subseteq \naive(SF)   \subseteq \cf(SF)
	\end{gather*}
	To extend the graph-related terminology to the directed hypergraph structure of SETAFs, we often take the \emph{primal graph}~\citep{DvorakKW21} as a starting point.
	Intuitively, collective attacks are ``split up'' in order to obtain a directed graph with a similar structure as the original SETAF.
	\begin{definition}[Primal Graph]
		Let $SF=(A,R)$ be a SETAF. Its \emph{primal graph} is defined as $\primal(SF)=(A,R')$ with $R'=\{(t,h)\mid (T,h)\in R, t\in T\}$.
	\end{definition}
	\begin{example}
		Recall the SETAF $SF$ from Example~\ref{ex:setaf}.
		Its primal graph $\primal(SF)$ looks as follows.
		
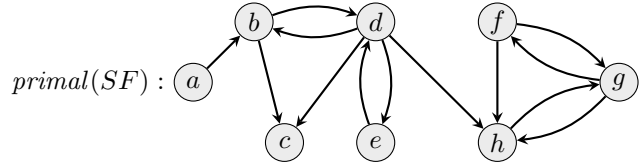
\begin{center}
			\begin{tikzpicture}[scale=0.8,>=stealth]
			\path 
			(-1,-1)node[arg,label={left:$\primal(SF):$}] (a){$a$}
			(0,0)node[arg] (b){$b$}
			(0.5,-2)node[arg] (c){$c$}
			(2,0)node[arg] (d){$d$}
			(2,-2)node[arg] (e){$e$}
			(4,0)node[arg] (f){$f$}
			(4,-2)node[arg] (h){$h$}
			(6,-1)node[arg] (g){$g$}
			;
			\path[thick,->]
			(a)edge(b)
			(b)edge(c)
			(d)edge(c)
			(d)edge(h)
			(f)edge(h)
			;
			\path[thick,->,bend left=20]
			(d)edge(b)
			(b)edge(d)
			(d)edge(e)
			(e)edge(d)
			(f)edge(g)
			(g)edge(f)
			(h)edge(g)
			(g)edge(h)
			;
			\end{tikzpicture}
		\end{center}
	\end{example}
	Finally, we introduce the notion of the \emph{projection},
	which we will revisit and redefine in Sections~\ref{sec:dir} and \ref{section:sccs}.
	\begin{definition}[Projection]\label{def:projection}
		Let $SF=(A,R)$ be a SETAF and
		$S\subseteq A$.
		We define the \emph{projection} $\proj{SF}{S}$ of $SF$ w.r.t.\ $S$ as
		$(S,\{ (T',h) \mid (T,h)\in R, h\in S, T'=T\cap S, T'\not=\emptyset  \})$.
	\end{definition}
	
	\section{Basic Principles}\label{section:setaf principles}
	We start our principles-based analysis of SETAF semantics by generalizing 
	basic principles from AFs. Satisfaction (or non-satisfaction) of principles allows us to distinguish semantics 
	with respect to fundamental properties that are crucial in certain applications.
	\begin{table}[]
		\small
		\setlength\tabcolsep{2pt}
		\def\arraystretch{1.05}
		\centering
		\begin{tabular}{|l|c|c|c|c|c|c|c|c|c|c|c|c|}
			\hline
			& $\grd$ & $\adm$ & $\com$ & $\stb$ & $\pref$ & $\naive$ & $\sem$ & $\stage$ \\ \hline
			Conflict-freeness      & \cm    & \cm    & \cm    & \cm    & \cm     & \cm      & \cm    & \cm      \\ \hline
			Defense                & \cm    & \cm    & \cm    & \cm    & \cm     & \xm      & \cm    & \xm      \\ \hline
			Admissibility          & \cm    & \cm    & \cm    & \cm    & \cm     & \xm      & \cm    & \xm      \\ \hline
			Reinstatement          & \cm    & \xm    & \cm    & \cm    & \cm     & \xm      & \cm    & \xm      \\ \hline
			CF-reinstatement       & \cm    & \xm    & \cm    & \cm    & \cm     & \cm      & \cm    & \cm      \\ \hline
			Naivety                & \xm    & \xm    & \xm    & \cm    & \xm     & \cm      & \xm    & \cm      \\ \hline
			I-maximality           & \cm    & \xm    & \xm    & \cm    & \cm     & \cm      & \cm    & \cm      \\ \hline
			Allowing abstention    & \cm    & \cm    & \cm    & \xm    & \xm     & \xm      & \xm    & \xm      \\ \hline
			Crash resistance       & \cm    & \cm    & \cm    & \xm    & \cm     & \cm      & \cm    & \cm      \\ \hline
			Modularization         & \cm    & \cm    & \cm    & \cm    & \cm     & \xm      & \cm    & \xm    	\\ \hline
			Directionality & \cm    & \cm    & \cm    & \xm    & \cm     & \xm      & \xm    & \xm      \\ \hline
			Non-interference       & \cm    & \cm    & \cm    & \xm    & \cm     & \cm      & \cm    & \cm      \\ \hline
			SCC-recursiveness      & \cm    & \cm    & \cm    & \cm    & \cm     & \xm      & \xm    & \xm      \\ \hline
			
		\end{tabular}
		\caption{An overview of our results regarding SETAF principles.}
		\label{table:principles}
	\end{table}
	The principles we consider have natural counterparts for Dung-style AFs, simply by applying them to SETAFs where $|T| = 1$ for each tail. 
	We therefore formalize the following observation:	
	\begin{observation}
		Let $P$ be a SETAF-principle that properly generalizes an AF-principle $P^{AF}$ in the sense that for SETAFs $SF$ with $|T|=1$ for each $(T,h)\in R(SF)$, every semantics $\sigma$ satisfies $P$ iff it satisfies $P^{AF}$. 
		Then if a semantics $\sigma$ does not satisfy $P^{AF}$, then $\sigma$ does not satisfy $P$.
	\end{observation}
	As all of our principles properly generalize the respective AF-principles,  
	whenever a principle is not satisfied for AFs, this translates to the corresponding SETAF principle as well.
	
	Next, we follow 
	\citep{TorreV17} 
	and introduce analogous principles for SETAFs. 
	Our first set of principles is concerned with basic properties of semantics. 
	\begin{definition}
		The following properties are said to hold for a semantics $\sigma$
		if the listed condition holds for each SETAF $SF=(A,R)$ and each set $E\in \sigma(SF)$.
		\begin{itemize}
			\item \emph{conflict-freeness}: $E\in\cf(SF)$ 
			\item \emph{defense}: each $a\in E$ is defended 
			\item \emph{admissibility}: $E\in \adm(SF)$ 
			\item \emph{reinstatement}: $a\in E$ for each $a$ defended by $E$
			\item \emph{CF-reinstatement}: $a\in E$ for each $a$ defended by $E$ s.t.\ $E\cup \{a\}\in\cf(SF)$
		\end{itemize}
	\end{definition}
	
	These principles formalize requirements which oftentimes do or do not hold immediately by definition of the given semantics. 
	We do not discuss these in detail here; the respective (non-)satisfaction results are reported in Table~\ref{table:principles}. 
	
	The next principles make statements about the structure of the extensions of a given semantics $\sigma$. 
	Here, I-maximality is due to \citeauthor{BaroniG07}~(\citeyear{BaroniG07}). 
	\begin{definition}
		The following properties are said to hold for a semantics $\sigma$
		if the listed condition holds for each SETAF $SF=(A,R)$ and each $E,E'\in \sigma(SF)$.
		\begin{itemize}
			\item \emph{naivety} iff there is no $E'\in\cf(SF)$ s.t.\ $E'\supset E$ 
			\item \emph{I-maximality} iff $E\subseteq E'$ implies $E=E'$
		\end{itemize}
	\end{definition}
	Whether the naivety principle is satisfied can be seen by closely inspecting the definition of the semantics. I-maximality results for SETAFs have been shown in~\citep{DvorakFW19}.
	
	The principle of allowing abstention can be attributed to \citeauthor{BaroniCG11}~(\citeyear{BaroniCG11}).
	\begin{definition}
		A semantics $\sigma$ satisfies \emph{allowing abstention} if 
		for all SETAFs $SF=(A,R)$, for all $a\in A(SF)$, if there exist $E,E'\in \sigma(SF)$
		with $a\in E$ and $a\in E'^+_R$ then there exists a $D\in \sigma(SF)$ such that $a\not\in D^\oplus_R$.
	\end{definition}
	
	Allowing abstention is satisfied by complete semantics, since---as in AFs---if there exist $E,E'\in \com(SF)$
	with $a\in E$ and $a\in E'^+$,
	this means $a\notin G^\oplus$ where $G\in \grd(SF)$.
	
	For the last principle we discuss within this section, which is due to~\citeauthor{CaminadaCD12}~(\citeyear{CaminadaCD12}), we require another notion.
	We call a SETAF $SF'=(A',R')$ 
	\emph{contaminating} for a semantics $\sigma$ if for every SETAF $SF=(A,R)$ with $A\cap A'=
	\emptyset$, it holds that $\sigma(SF\cup SF')=\sigma(SF')$, where $SF\cup SF'$ is the SETAF $(A\cup A',R\cup R')$.
	\begin{definition}[Crash resistance]
		A semantics $\sigma$ satisfies \emph{crash resistance} if 
		there is no contaminating SETAF for $\sigma$.
	\end{definition}
	As in the case for AFs, $\stb$ is the only semantics considered in this paper which is not crash-resistant. 
	The reason is that one can choose $SF'$ to be an isolated odd cycle, yielding $\stb( SF\cup SF' ) = \emptyset$ for any SETAF $SF$. 
	The other semantics are more robust in this regard and yield  
	$\sigma(SF \cup SF') = \{E\cup E' \mid E\in\sigma(SF),E'\in\sigma(SF')\}$ 
	whenever $A\cap A' = \emptyset$. 
	
	In the following sections we will introduce and investigate 
	further 
	principles regarding
	computational properties. 
	
	\section{Reduct and Modularization}\label{section:reduct}
	
	In this section, we will generalize the modularization property \citep{baumann2020comparing}, which yields concise alternative characterizations for the classical semantics in AFs, to SETAFs. 
	As a first step, we require the so-called reduct of a SETAF.
	
	\subsection{The SETAF Reduct}
	In the remainder of this paper, the \emph{reduct} of a SETAF w.r.t.\ a given set $E$ will play a central role. Intuitively, the reduct w.r.t.\ $E$ represents the SETAF that result from ``accepting'' $E$ and rejecting what is defeated now, while not deciding on the remaining arguments.
	To illustrate the idea, consider the following example:
	\begin{example}
		\label{ex:running example reduct}
		Recall the SETAF $SF$ from Example~\ref{ex:setaf}.
		Consider the singleton 
		$\{a\}$. 
		If we view $a$ as accepted, then $b$ is rejected. This means that the attack from $b$ to $d$ can be disregarded. However, we also observe that $c$ cannot be attacked anymore since attacking it requires both $b$ and $d$. 
		Now consider 
		$\{f\}$. 
		Interpreting $f$ as accepted renders $g$ rejected. In addition, in order to attack $h$ only one additional argument (namely $d$) is required. 
		Thus, if we let 
		$E= \{a,f\}$, 
		then we expect the SETAF $SF^E$ --with the intuitive meaning that $a$ and $f$ are set to true-- to look as follows. 
		
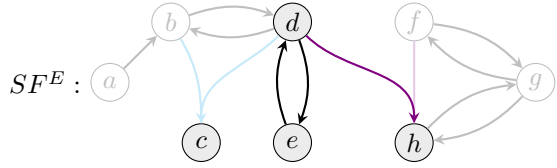
\begin{center}
			\begin{tikzpicture}[scale=0.8,>=stealth]
			\path 
			(-1,-1)node[argd,label={left:$SF^E:$}] (a){$a$}
			(0,0)node[argd] (b){$b$}
			(0.5,-2)node[arg] (c){$c$}
			(2,0)node[arg] (d){$d$}
			(2,-2)node[arg] (e){$e$}
			(4,0)node[argd] (f){$f$}
			(4,-2)node[arg] (h){$h$}
			(6,-1)node[argd] (g){$g$}
			;
			\path[thick,->]
			(a)edge[gray!50](b)
			(b)edge[color=white!80!cyan,out=-60,in=90,-](c)
			(d)edge[color=white!80!cyan,out=225,in=90](c)
			(f)edge[color=white!80!violet,out=-90,in=90,-](h) 
			(d)edge[color=violet,out=-45,in=90](h)
			;
			\path[thick,->,bend left=20]
			(d)edge[gray!50](b)
			(b)edge[gray!50](d)
			(d)edge(e)
			(e)edge(d)
			(f)edge[gray!50](g)
			(g)edge[gray!50](f)
			(h)edge[gray!50](g)
			(g)edge[gray!50](h)
			;
			\end{tikzpicture}
		\end{center}
		As we can see, the above depicted SETAF reflects the situation after $E$ is set to true: e.g.\ $c$ is defended and in order to defeat $h$, only $d$ is required. 
	\end{example}
	That is, in the reduct $SF^E$, we only need to consider arguments that are still
	undecided, \ie all arguments neither in $E$ nor attacked by $E$.
	As illustrated in the example, some attacks that involve deleted arguments are preserved which is in contrast to the AF-reduct~\citep{baumann2020comparing}.
	In particular, if the arguments in the tail of an attack are ``accepted'' (\ie in $E$),
	the attack can still play a role in attacking or defending.
	If the tail of an attack $(T,h)$ is already attacked by $E$,
	we can disregard $(T,h)$.
	\begin{definition}
		Given a SETAF $SF=(A,R)$ and $E\subseteq A$, the \emph{$E$-reduct of $SF$} is the SETAF $SF^E=(A' , R')$, with
		\begin{align*}
		A'\enspace=\enspace& A \setminus E^\oplus_R\\
		R'\enspace=\enspace& \{(T\setminus E, h) \mid (T,h)\in R, \,
		T\cap E^+_R=\emptyset,\\  
		&\qquad\qquad\qquad\,  T \not\subseteq E,\,
		h\in A'
		\}
		\end{align*}
	\end{definition}
	Thereby, the condition $T\cap E^+_R = \emptyset$ captures cases like the attack $(\{b,d\},c)$ from our example: $b$ is 
	attacked by $E$, and thus, the whole attack gets removed. 
	The reason why we take $(T\setminus E,h)$ as our attacks is the partial evaluation as in the attack $(\{d,f\},h)$ after setting $f$ to true: only $d$ is now left required in order to ``activate'' the attack against $h$. 
	\begin{example}
		Given the SETAF $SF$ from Example~\ref{ex:running example reduct} as well as 
		$E=\{a,f\}$
		as before, the reduct $SF^E$ is the SETAF depicted above, \ie 
		$SF^E = \{ A', R' \}$
		with
		$A' 
		= \{c,d,e,h\}$ 
		and 
		$R' 
		= \{ (d,e), (e,d), (d,h) \}$.
	\end{example}
	We start our formal investigation of the reduct with a technical lemma to settle some basic properties.
	\begin{lemma}\label{prop:reductbasics setafs}
		Given a SETAF $\F = (A,R)$ and two disjoint sets $E,E'\subseteq A$. Let $\F^E = (A',R')$. 
		\begin{enumerate}
			\item If there is no $S\subseteq A$ s.t.\ $S\att_R E'$, then the same is true in $\F^E$.
			\item Assume $E$ does not attack $E'\in\cf(\F)$. Then, $E$ defends $E'$ iff there is no $S'\subseteq A'$ s.t.\ $S'\att_{R'} E'$. 
			\item Let $E\in\cf(\F)$. If $E\cup E'$ does not attack $E$ in $\F$ and $E'\subseteq A'$, with $E'\in\cf\left(\F^E\right)$ then $E\cup E'\in\cf(\F)$.
			\item Let $E\cup E'\in\cf(\F)$. If $E'\att_{R'} a$, then $E\cup E'\att_R a$.
			\item If $E\cup E'\in\cf(\F)$, then $\F^{E\cup E'} = \left(\F^E\right)^{E'}$.
		\end{enumerate}
	\end{lemma}
	
	\subsection{The Modularization Property}
	Having established the basic properties of the SETAF reduct, we are now ready to introduce the modularization property~\citep{baumann2020comparing}. 
	\begin{definition}[Modularization]
		A semantics $\sigma$ satisfies \emph{modularization} if 
		for all SETAFs $SF$, for every $E\in \sigma(SF)$ and $E'\in\sigma(SF^E)$, we have $E\cup E'\in \sigma(SF)$.
	\end{definition}
	Modularization allows us to build 
	extensions iteratively.
	After finding such a set $E\subseteq A$ we can efficiently compute its reduct $SF^E$ and
	pause before computing an extension $E'$ for the reduct in order 
	to obtain a larger extension $E \cup E'$ for $SF$.
	Hence, this first step can be seen as an intermediate result
	that enables us to reduce the computational effort of finding
	extensions in $SF$, as the arguments whose status is already determined by accepting $E$
	do not have to be considered again. Instead, we can reason on the reduct $SF^E$ 
	(see
	Section~\ref{sec:incremental}). 
	In the following, we establish the modularization property for admissible and complete semantics. 
	\begin{theorem}[Modularization Property]\label{theorem:modularization}
		Let $\F$ be a SETAF,
		$\sigma\in\{\adm,\com\}$ and $E\in\sigma(\F)$.
		\begin{enumerate}
			\item If $E'\in\sigma( \F^E )$, then $E\cup E'\in\sigma(\F)$. 
			\item If $E\cap E' = \emptyset$ and $E\cup E'\in\sigma(\F)$, then $E'\in\sigma(\F^E)$. 
		\end{enumerate}
	\end{theorem}
	\begin{proof}(for $\sigma = \adm$)
		Let $\F^E = (A',R')$.
		
		1) Since $E$ is admissible and $E'\subseteq A'$, $E'$ does not attack $E$. 
		By Lemma~\ref{prop:reductbasics setafs}, item 3, $E\cup E'\in\cf(\F)$. 
		Now assume $S\at_R E\cup E'$. 
		If $S\at_R E$, then $E\att_R S$ by admissibility of $E$. 
		If $S\at_R E'$, there is $T\subseteq S$ s.t.\ $(T,e')\in R$ for some $e'\in E'$. 
		In case $E\at_R T$, we are done. 
		Otherwise, $(T\setminus E,e')\in R'$ and by admissibility of $E'$ in $\F^E$, $E'\att_{R'} T\setminus E$. 
		By Lemma~\ref{prop:reductbasics setafs}, item 4, $E\cup E'\at_R T\setminus E$. 
		
		2) Now assume $E\cup E'\in\adm\left(\F\right)$. 
		We see $E'\in\cf\left(\F^E\right)$ as follows: 
		If $(T',e')\in R'$ for $T'\subseteq E'$ and $e'\in E'$, then there is some $(T,e')\in R$ with $T' = T\setminus E$. 
		Hence $E\cup E'\att_R E'$, contradiction.
		Now assume $E'$ is not admissible in $\F^E$, i.e.\ there is $(T',e')\in R'$ with $e'\in E'$ and $E'$ does not counterattack $T'$ in $\F^E$. 
		Then there is some $(T,e')\in R$ with $T' = T\setminus E$ and $T\cap E_R^+ = \emptyset$. 
		By admissibility of $E\cup E'$, $E\cup E'\att_R T$, say $(T^*,t)\in R$, $T^*\subseteq E\cup E'$ and $t\in T$. 
		Since $E\cup E'$ is conflict-free, $T^*\cap E_R^+ = \emptyset$ and thus we either have
		a) $T^*\subseteq E$, contradicting $T\cap E_R^+ = \emptyset$, or
		b) $(T^*\setminus E,t)\in R'$ and $t\in T'$, i.e. $E'$ counterattacks $T'$ in $\F^E$ contradicting the above assumption.
	\end{proof}
	The result for complete semantics can be obtained by using 
	Lemma~\ref{prop:reductbasics setafs}, items 2 and 5.
	Note that the modularization property also holds for $\stable$, $\pref$, and $\sem$ semantics.
	However,
	the only admissible set in the reduct w.r.t.\ a stable/preferred/semi-stable extension
	is the empty set, rendering the property trivial.
	The exact relation is captured by the following alternative characterizations of
	the semantics under our consideration.
	
	\begin{proposition}\label{prop:classical semantics and reduct SETAFs}
		Let $\F=(A,R)$ be a SETAF, $E\in\cf(\F)$ and $\F^E = (A',R')$.
		\begin{enumerate}
			\item $E\!\in\!\stb(\F)$ iff $SF^E = (\emptyset,\emptyset)$,
			\item  $E\!\in\!\adm(\F)$ iff $S\att_R E$ implies $S\setminus E \not\subseteq A'$, 
			\item   $E\!\in\!\prf(\F)$ iff $E\in\adm(\F)$ and $\bigcup \adm \left(\F^E\right) = \emptyset$,
			\item  $E\!\in\!\com(\F)$ iff $E\in\adm(\F)$ and no argument in $SF^E$ is unattacked,
			\item  $E\!\in\!\sem(\F)$ iff $E\in\pref(\F)$ and there is no $E'\in \pref(SF)$ such that $A(SF^{E'})\subset A(SF^E)$.
		\end{enumerate}
	\end{proposition}
	\begin{proof}
		The characterizations for $\stb$ and $\adm$ are straightforward and $\pref$ is due to the modularization property of $\adm$. 
		For $\com(\F)$ we apply Lemma~\ref{prop:reductbasics setafs}, item 2, to each singleton $E'$ occurring in $\F^E$. For $\sem$ recall that range-maximal preferred extensions are semi-stable.
	\end{proof}
	From the characterization of complete semantics provided in Proposition~\ref{prop:classical semantics and reduct SETAFs} we infer that for any SETAF $SF$ the complete extensions $E\in\com(SF)$ satisfy $\grd(\F^E) = \{ \emptyset \}$ implying modularization for $\grd$. 
	Moreover, as the grounded extension $G$ is the least complete, we can utilize modularization of $\adm$ and obtain $G$
	by the following procedure:
	(1)~add the set of unattacked arguments $U$ into $G$, (2) repeat step (1) on $SF^{U}$ until there are no unattacked arguments.
	
	\section{Directionality and Non-Interference}\label{sec:dir}
	In this section we discuss the principles directionality and non-interference.
	Intuitively, these principles give information about the behavior of ``separate''
	parts of a framework.  In order to formalize this separation-property, we
	start of with the notion of unattacked sets of arguments\footnote{While in the previous section we used ``unattacked arguments'', i.e.\ arguments that are not the head of any attack, unattacked \emph{sets of arguments} allow for attacks within the set.}.
	For directionality~\citep{BaroniG07} we have to carefully consider this notion
	in order to obtain a natural generalization of the AF case preserving the intended meaning.
	A naive definition of unattacked sets will lead to nonsensical results:
	assume a set $S$ is unattacked in a SETAF $SF=(A,R)$ whenever 
	it is not attacked from ``outside'', i.e.\ if the condition 
	$A\setminus S \not\att_R S$ holds.
	\begin{example}\label{ex:directionality}
		Consider now the following SETAF (a) and its projections (b), (c) w.r.t.\ the ``unattacked'' set $S=\{a,c\}$.
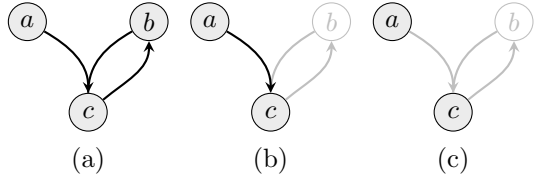
		\begin{center}
			\begin{tikzpicture}[scale=0.8,>=stealth]
			\path 
			(0,0)node[arg] (a){$c$}
			(1,1.5)node[arg] (b){$b$}
			(-1,1.5)node[arg] (c){$a$}
			;
			\path[thick,->]
			(b)edge[in=90,out=-150,-](a)
			(c)edge[in=90,out=-30](a)
			(a)edge[in=-90,out=40](b);
			\path
			(0,-0.8)node(lab1){(a)}
			(3,-0.8)node(lab1){(b)}
			(6,-0.8)node(lab1){(c)}
			;
			\path 
			(3,0)node[arg] (a1){$c$}
			(4,1.5)node[argd] (b1){$b$}
			(2,1.5)node[arg] (c1){$a$}
			;
			\path[thick,->]
			(b1)edge[in=90,out=-150,-,gray!50](a1)
			(c1)edge[in=90,out=-30](a1)
			(a1)edge[in=-90,out=40,gray!50](b1);
			\path 
			(6,0)node[arg] (a2){$c$}
			(7,1.5)node[argd] (b2){$b$}
			(5,1.5)node[arg] (c2){$a$}
			;
			\path[thick,->]
			(b2)edge[in=90,out=-150,-,gray!50](a2)
			(c2)edge[in=90,out=-30,gray!50](a2)
			(a2)edge[in=-90,out=40,gray!50](b2);
			
			\end{tikzpicture}
		\end{center}%
		Note that
		$\{a,c\}$ is stable
		in (a).
		If we now consider the projection $\proj{SF}{S}$---see (b)---we find that
		$\{a,c\}$ is not stable, falsifying directionality.
		However, one might argue that this is due to the credulous nature of our projection-notion.
		We could easily consider a different proper generalization of the projection,
		namely 
		${SF}{\downarrow}^*_{S}=(S,\{ (T,h) \mid (T,h)\in R, T\cup \{h\}\subseteq S \})$.
		In this more skeptical version we delete attacks if any of the arguments in the tail are not in the projected set---see (c).
		However, we still cannot obtain the desired results:
		in (a) we find $\{a\}$ to be the unique grounded extension,
		while in (c) $\{a,c\}$ is grounded, again falsifying directionality.
		As for the directionality principle we do not want to add additional arguments or attacks and we exhausted all possible reasonable projection notions for this small example,
		we conclude that the underlying definition of \emph{unattacked} sets was improper.
		We therefore suggest a different definition---and at the same time suggest to think of these sets rather as ``uninfluenced'' than ``unattacked''.
		In AFs, clearly both notions coincide.
		However, we still argue that the concept of ``influence'' captures the true nature of directionality in a more intuitive and precise manner.
		Moreover, note that in the case of uninfluenced sets both notions of projection coincide, as well as the
		notion of
		\emph{restriction} (see Definition~\ref{def:restriction}) for arbitrary sets $D\subseteq A\setminus S$.
	\end{example}%
	\begin{definition}[Influence]
		Let $SF=(A,R)$ be a SETAF.
		An argument $a\in A$ \emph{influences} $b\in A$
		if there is a directed path from $a$ to $b$ in $\primal(SF)$.
		A set $S\subseteq A$ is \emph{uninfluenced} in $SF$
		if no $a\in A\setminus S$ influences any $b\in S$. 
		We denote the set of uninfluenced sets by $\USets(SF)$.
	\end{definition}

	Utilizing this notion, we can properly generalize directionality~\citep{BaroniG07}.
	\begin{definition}[Directionality]
		A semantics $\sigma$ satisfies \emph{directionality} if 
		for all SETAFs $SF$ and every $S\in \USets(SF)$ it holds $\sigma(\proj{SF}{S})=\{E\cap S \mid E\in \sigma(SF)\}$.
	\end{definition}
	We will revisit directionality at the end of the next section,
	as we can utilize SCC-recursiveness to show that $\grd$, $\comp$, and $\pref$ satisfy directionality.
	Similarly, we generalize \emph{non-interference}~\citep{CaminadaCD12},
	which has an even stronger requirement.
	$S\subseteq A$ is \emph{isolated} in $SF=(A,R)$, if $S$ is uninfluenced and $A\setminus S$ is uninfluenced,
	i.e.\ there are no edges in $\primal(SF)$ between $S$ and $A\setminus S$. 
	\begin{definition}[Non-interference]
		A semantics $\sigma$ satisfies \emph{non-interference} iff for all SETAFs $SF$ and all isolated $S\subseteq A(SF)$, 
		it holds $\sigma(\proj{SF}{S})=\{E\cap S\mid E\in \sigma(SF)\}$.
	\end{definition}
	Clearly, directionality implies non-interference.
	It is easy to see from the respective definitions that also naive, semi-stable, and stage semantics satisfy non-interference.
	\section{SCC-Recursiveness}
	\label{section:sccs}
	SCC-recursiveness~\citep{BaroniGG05a} can
	be seen as a tool to iteratively compute extensions, or an alternative characterization of many semantics.
	SCC-recursiveness formalizes the intuition that the acceptance status of an argument depends
	only on its ancestors---i.e., the arguments that feature a directed path to the argument in question.	
	In Section~\ref{sec:dir} we captured this concept with the notion of \emph{influence}.
	In a nutshell, an argument $a$ ``influences'' an argument $b$ in a SETAF $SF$
	if there is a directed path from $a$ to $b$ in $\primal(SF)$.
	It is therefore reasonable to investigate SCCs (strongly connected components)
	with this idea in mind.
	In particular, our definition of SCCs captures the equivalence classes w.r.t.\ the influence relation.
	\begin{definition}[SCCs]
		Let $SF$ be a SETAF. By $\SCCs(SF)$ we denote the set of strongly connected components of $SF$, which we define as the sets of arguments contained in the strongly connected components of $\primal(SF)$.
	\end{definition}
	
	\begin{example}\label{ex:scc}
		Recall our SETAF from before. 
		\begin{center}
			\begin{tikzpicture}[scale=0.8,>=stealth]
			\path (-1,-1) node[draw,thick,dashed,color=black!60,fill=white!0,minimum size=0.9cm,circle](s1){}
			;
			\path (0.5,-2) node[draw,thick,dashed,color=black!60,fill=white!0,minimum size=0.9cm,circle](s2){}
			;
			\path 
			(-1,-1)node[arg] (a){$a$}
			(0,0)node[arg] (b){$b$}
			(0.5,-2)node[arg] (c){$c$}
			(2,0)node[arg] (d){$d$}
			(2,-2)node[arg] (e){$e$}
			(4,0)node[arg] (f){$f$}
			(4,-2)node[arg] (h){$h$}
			(6,-1)node[arg] (g){$g$}
			;
			\path[thick,->]
			(a)edge(b)
			(b)edge[color=cyan,out=-60,in=90,-](c)
			(d)edge[color=cyan,out=225,in=90](c)
			(d)edge[color=violet,out=-45,in=90](h)
			(f)edge[color=violet,out=-90,in=90,-](h)
			;
			\path[thick,->,bend left=20]
			(d)edge(b)
			(b)edge(d)
			(d)edge(e)
			(e)edge(d)
			(f)edge(g)
			(g)edge(f)
			(h)edge(g)
			(g)edge(h)
			;
			\draw [thick,dashed,color=black!60] plot [mark=none, smooth cycle] coordinates {(-0.4,0.4) (2.4,0.4) (2.4,-2.4) (1.6,-2.4) (1.2,-0.8) (-0.4,-0.4)};
			\draw [thick,dashed,color=black!60] plot [mark=none, smooth cycle] coordinates {(3.7,0.55) (6.7,-1.)  (3.7,-2.55) };
			
			\end{tikzpicture}
		\end{center}
		In this SETAF, we have the four SCCs
		$\{a\}, \{b,d,e\}, \{c\}$, and $\{f,g,h\}$, 
		as depicted in dashed lines.  
	\end{example}
	Analogously to \citep{BaroniGG05a}, we partition the arguments in defeated, provisionally defeated and undefeated ones.
	Intuitively, accepting a defeated argument would lead to a conflict,
	the provisionally defeated cannot be defended and will therefore be rejected (while not being irrelevant for defense of other arguments),
	and the undefeated form the candidates for extensions. We obtain the following formal definition of the sets we just described. 
	\begin{definition}
		Let $SF=(A,R)$ be a SETAF.
		Moreover, let $E\subseteq A$ be a set of arguments and $S\in \SCCs(SF)$ be an SCC.
		We define the set of 
		defeated arguments $\Dscc{SF}{S}{E}$, 
		provisionally defeated arguments $\Pscc{SF}{S}{E}$, and 
		undefeated arguments $\Uscc{SF}{S}{E}$ 
		w.r.t.\ $S,E$ as 
		\begin{align*}
		\Dscc{SF}{S}{E} &= \{a\in S \mid E\setminus S \att_R a \}, \\
		\Pscc{SF}{S}{E} &= \{a\in S \mid A\setminus (S\cup E^+)\!\att_R\! a \}
		\!\setminus\!\Dscc{SF}{S}{E}, \\
		\Uscc{SF}{S}{E} &= S\setminus (\Dscc{SF}{S}{E}\cup\Pscc{SF}{S}{E}).
		\end{align*}
		Moreover, we set 
		$\UPscc{SF}{S}{E} = \Uscc{SF}{S}{E}\cup \Pscc{SF}{S}{E}$.
	\end{definition}

	In order to properly investigate SCC-recursiveness, we need the notion of the \emph{restriction}.
	The restriction coincides with the \emph{projection} on AFs~(cf.~Definition~\ref{def:projection}).
	However, in the following we will argue that
	the projection does not capture the intricacies of this process.
	Ultimately, we will see that for a reasonable restriction
	we need semantic tools that are similar to the reduct.
	For that, we revisit Example~\ref{ex:directionality}.
	
	\begin{example}\label{ex:scc_proj}
		Consider the following SETAFs (a) and (b).
		\begin{center}
			\begin{tikzpicture}[scale=0.8,>=stealth]
			\path 
			(0,0)node[arg] (a){$c$}
			(1,1.5)node[arg] (b){$b$}
			(-1,1.5)node[arg] (c){$a$}
			;
			\path[thick,->]
			(b)edge[in=90,out=-150,-](a)
			(c)edge[in=90,out=-30](a)
			(a)edge[in=-90,out=40](b);
			\path
			(0,-0.8)node(lab1){(a)}
			(4,-0.8)node(lab1){(b)}
			;
			\path 
			(4,0)node[arg] (a1){$c$}
			(5,1.5)node[arg] (b1){$d$}
			(3,1.5)node[arg] (c1){$b$}
			(2.5,0)node[arg] (d){$a$}
			(5.5,0)node[arg] (e){$e$}
			;
			\path[thick,->]
			(b1)edge[in=90,out=-150](a1)
			(c1)edge[in=90,out=-30](a1)
			(d)edge(c1)
			(a1)edge(e)
			(e)edge[loop,looseness=5,in=25,out=-25](e)
			(e)edge(b1);
			
			\end{tikzpicture}
		\end{center}
		In (a), assume we accept the argument $a$.
		Now for the remaining SCC $\{b,c\}$ the projection $\proj{SF}{\{b,c\}}$
		yields the attacks $(b,c)$ and $(c,b)$, as one might expect.
		In (b), assume we accept $a$ and therefore reject $b$.
		The projection $\proj{SF}{\{c,d,e\}}$
		yields a cycle of length $3$, and none of the remaining arguments can be accepted.
		However, as $c$ is defended this is not the expected behavior.
		One might argue that this notion of projection is too credulous, i.e.,
		attacks remain that have to be discarded.
		Recall Example~\ref{ex:directionality}
		where we defined an alternative projection, namely
		${SF}{\downarrow}^*_{S}=(S,\{ (T,h) \mid (T,h)\in R, T\cup \{h\}\subseteq S \})$.
		Now, one can check that we get the expected results in (b).
		However, in (a) ${SF}{\downarrow}^*_{\{b,c\}}$ only features the attack $(c,b)$, and it incorrectly seems like we cannot accept $b$.
		We solve this problem by adapting the notion of \emph{restriction}
		such that both cases are handled individually.
		We keep track of a set of rejected arguments and discard attacks
		once an argument in its tail is discarded---these attacks are irrelevant to the further evaluation of the SETAF.	
	\end{example}
	\begin{definition}[Restriction]\label{def:restriction}
		Let $SF=(A,R)$ be a SETAF and let $S,D\subseteq A$.
		We define the \emph{restriction} of $SF$ w.r.t.\ $S$ and $D$ as the SETAF 
		$\projIII{SF}{D}{S} = (S',R')$
		where
		\begin{align*}
		S' &= (A\cap S)\setminus D  \\ 
		R' &= \{ (T\cap S',h) \mid (T,h)\in R, h\in S',	T\cap D =\emptyset,\\
		&\qquad\qquad\qquad\qquad\,   T\cap S'\neq\emptyset  \}.
		\end{align*}
	\end{definition}
	The restriction handles both cases of Example~\ref{ex:scc_proj} according to our intuition.
	In (a) the SETAF $\projIII{SF}{\emptyset}{\{b,c\}}$ contains $b$ and $c$, and as we accepted the part tail of $(\{a,b\},c)$ outside $\{b,c\}$ (namely $a$),
	the attack $(b,c)$ is kept.
	In (b) $\projIII{SF}{\{b\}}{\{c,d,e\}}$ contains the ``attack-chain'' $(c,e),(e,d)$, and as the tail of $(\{b,d\},c)$ is already defeated, we disregard $(d,c)$.
	We want to emphasize that this example illustrates how the notion of projection is akin to the SETAF-reduct: 
	indeed, constructing $\projIII{SF}{D}{S}$ consists in projecting to a certain set of arguments and then i) removing attacks where defeated arguments are involved as well as 
	ii) partially evaluating the remaining tails. 
	Formally, the connection is as follows.
	\begin{lemma}\label{lemma:reduct_restriction}
		Let $SF=(A,R)$ be a SETAF and let $E,S\subseteq A$.
		Then\\
		$\projIII{SF}{(E\setminus S)^+}{S} = \projI{ SF^{(E\setminus S)}}{S}$.
	\end{lemma}
	Let us now formally introduce SCC-recursiveness~\citep{BaroniGG05a} as a SETAF principle.
	Extensions are recursively characterized as follows:
	if the SETAF $SF$ consists of a single SCC, the \emph{base function $\BF$} of the semantics
	yields the extensions.
	For SETAFs that consist of more SCCs we apply the \emph{generic selection function $\GF$},
	where $SF$ is evaluated separately on each SCC, taking into account arguments that are defeated by previous SCCs
	\begin{definition}[SCC-recursiveness]\label{principle:scc1}
		A semantics $\sigma$ satisfies \emph{SCC-recursiveness} if 
		for all SETAFs $SF=(A,R)$,
		it holds that $\sigma(SF)=\GF(SF)$,
		where $\GF(SF)\subseteq 2^A$ 
		is defined as: $E\subseteq A\in \GF(SF)$ if and only if
		\begin{itemize}
			\item if $|\SCCs(SF)|=1$, $E\in \BF(SF)$,
			\item otherwise, $\forall S\in \SCCs(SF)$ it holds $E\cap S \in \GF({\projIII{SF}{(E\setminus S)^+}{\UPscc{SF}{S}{E}}})$,
		\end{itemize}
		where $\BF$ 
		is a function that maps a SETAF $SF=(A,R)$ with $|\SCCs(SF)|=1$
		to a subset of $2^A$.
	\end{definition}
	In the following subsections we will investigate and refine SCC-recursiveness
	for the different semantics under our consideration.
	For the proofs we loosely follow the structure of~\citep{BaroniGG05a},
	incorporating our SETAF-specific notions.
	\subsection{Stable Semantics}
	We start with stable semantics, as this is the easiest case.
	\begin{example}
		Recall Example~\ref{ex:scc} and consider the stable extension 
		$E = \{ a, c, d, f \}$
		of $SF$. 
		Let 
		$S = \{b,d,e\}$. 
		Indeed, 
		$E\cap S = \{d\}$
		is a stable extension of the projected SETAF 
		$\projIII{SF}{(E\setminus S)^+}{S} =
		(\{d,e\}, \{(d,e),(e,d)\})$. 
	\end{example}
	In this section we will show that this is no coincidence, \ie $\stb$ satisfies SCC-recursiveness. 
	For the investigation of SCC-recursiveness in stable semantics
	we use the fact that there are no undecided arguments. 
	Thus, in each step we do not have to keep track of as much information from previous SCCs. 
	Formally, we obtain the following auxiliary lemma. 
	\begin{lemma}
		\label{lem:noPinStable}
		Let $SF$ be a SETAF and
		$E\!\in\! \stb(SF)$, then
		for all $S\!\in\! \SCCs(SF)$ it holds $\Pscc{SF}{S}{E}\!=\!\emptyset$.
	\end{lemma}
	\begin{proof}
		For an argument $a$ to be provisionally defeated,
		there has to be an attack $(T,a)$ with $T\not\subseteq (E\cup E^+)$,
		a contradiction to the requirement of stable extensions.
	\end{proof}
	We continue with the main technical underpinning for the SCC-recursive characterization of stable semantics.
	Intuitively, Proposition~\ref{lemma:sccStableKey} states that an extension $E$ is ``globally'' stable in $SF$
	if and only if for each of its SCCs $S$ it is ``locally'' stable
	in $\projIII{SF}{(E\setminus S)^+}{\UPscc{SF}{S}{E}}$.
	\begin{proposition}
		\label{lemma:sccStableKey}
		Let $SF=(A,R)$ be a SETAF and let $E \subseteq A$, then
		$E\in \stb(SF)$ if and only if $\forall S\in \SCCs(SF) $ it holds
		$(E\cap S) \in \stb(\projIII{SF}{(E\setminus S)^{+}}{\UPscc{SF}{S}{E}})$.
	\end{proposition}
	\begin{proof}
		Let $SF'=\pte{S}$ be the ``local'' SETAF w.r.t.\ an arbitrary SCC $S\in \SCCs(SF)$.
		If we assume $E$ is globally stable,
		we need to show that
		(1) $(E\cap S)$ is contained in the local SETAF, i.e.\ $E\cap S\subseteq \UPscc{SF}{S}{E}$,
		(2) $(E\cap S)$ is locally conflict-free,
		(3) $(E\cap S)$ (locally) attacks all arguments not in it.
		For 1.\ note that by global conflict-freeness there cannot be an $a\in \Dscc{SF}{S}{E}\cap E$.
		Also 2.\ follows from global conflict-freeness, as any violation of local conflict-freeness by an attack $(T,h)\in R(SF')$ would imply the existence of an attack $(T',h)\in R$ with $T\subseteq T' \subseteq E$.
		For 3.~consider an arbitrary argument $a\in A(SF')\setminus E$.
		As $E$ is globally stable, we know $E\att_R a$.
		Moreover, as we assume $a\in A(SF')$ we know $a$ is not defeated, i.e.\ (parts of) an attack towards $a$ appears in $R(SF')$, establishing that $(E\cap S)$ is locally stable.
		For the other direction, assume $E$ is locally stable for every SCC.
		Global conflict-freeness follows from local conflict-freeness,
		as any attack $(T,h)$ with $T\cup \{h\}\subseteq E$ would imply a conflict in the SCC of $h$.
		It remains to show that $E\att_R a$ for all $a\in A\setminus E$.
		For $a$'s SCC $S$, either $a\in \Dscc{SF}{S}{E}$ (and we are done)
		or $a\in \UPscc{SF}{S}{E}$, in which case it is locally attacked,
		and by construction of the restriction globally attacked.		
	\end{proof}
	This leads us the the characterization of stable extensions. 
	The base function is $\stb(SF)$, the base case follows immediately.
	The composite case follows from~Proposition~\ref{lemma:sccStableKey}.
	\begin{theorem}
		Stable semantics is SCC-recursive.
	\end{theorem}
	
	\subsection{Admissible Sets}
	As already mentioned, when investigating stable semantics we can use the observation that each argument is either in $E$ or defeated by $E$. 
	For admissibility-based semantics, undecidedness of arguments is also possible which we need to handle in our SCCs as well. 
	This is reflected in the definition due to an added second component of $\GF$ which intuitively collects all arguments that can still be defended within the current SCC.
	We account for this in Definition~\ref{def:scc2}
	by maintaining a set of candidate arguments $C$.
	Moreover,
	the particular case of SETAFs gives rise to a novel scenario, where certain attacks are present in an
	SCC, but  not applicable. 
	\begin{example}\label{example:mitigated}
		Recall our SETAF from before. 
		
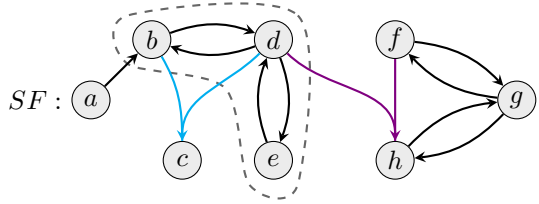
\begin{center}
			\begin{tikzpicture}[scale=0.8,>=stealth]
			\draw [thick,dashed,color=black!60] plot [mark=none, smooth cycle] coordinates {(-0.4,0.4) (2.4,0.4) (2.4,-2.4) (1.6,-2.4) (1.2,-0.8) (-0.4,-0.4)};
			\path 
			(-1,-1)node[arg,label={left:$SF:$}] (a){$a$}
			(0,0)node[arg] (b){$b$}
			(0.5,-2)node[arg] (c){$c$}
			(2,0)node[arg] (d){$d$}
			(2,-2)node[arg] (e){$e$}
			(4,0)node[arg] (f){$f$}
			(4,-2)node[arg] (h){$h$}
			(6,-1)node[arg] (g){$g$}
			;
			\path[thick,->]
			(a)edge(b)
			(b)edge[color=cyan,out=-60,in=90,-](c)
			(d)edge[color=cyan,out=225,in=90](c)
			(d)edge[color=violet,out=-45,in=90](h)
			(f)edge[color=violet,out=-90,in=90,-](h)
			;
			\path[thick,->,bend left=20]
			(d)edge(b)
			(b)edge(d)
			(d)edge(e)
			(e)edge(d)
			(f)edge(g)
			(g)edge(f)
			(h)edge(g)
			(g)edge(h)
			;
			\end{tikzpicture}
		\end{center}
		Let $E = \emptyset$. 
		First consider $S = \{b,d,e\}$. 
		Then 
		\begin{align*}
		\Dscc{SF}{S}{E} &=  \emptyset && \pte{S} = (S',R')\\
		\Pscc{SF}{S}{E} &= \{b \}  &&  S' = \{b,d,e\}\\
		\Uscc{SF}{S}{E} &= \{d,e \} && R' = \{(b,d),(d,b),(d,e),(e,d)\}
		\end{align*}
		In contrast to the situation we saw for stable semantics we now also need to bear in mind that $b$ cannot be defended by arguments in $S$ (due to $a$). 
		Hence the additional information in $\GF$ is required. 
		Now let $T = \{f,g,h\}$. 
		Then 
		\begin{align*}
		\Dscc{SF}{T}{E} &=  \emptyset && \pte{T} = (T',R'')\\
		\Pscc{SF}{T}{E} &= \emptyset  &&  T' = \{f,g,h\}\\
		\Uscc{SF}{T}{E} &= \{f,g,h \} && R'' = \{(f,h),(g,h),(h,g),\\
		&&&(f,g),(g,f)\}
		\end{align*}
		We now observe that although there is an attack from $f$ to $h$ in $\projIII{SF}{\emptyset}{T}$,
		the argument $h$ can actually not be defeated by $f$, because this would require $d$ to be present in our extension.
		Note however that we cannot delete the attack $(f,h)$,
		as this would mean we could accept $h$---without defending $h$ against the attack from $\{d,f\}$.
		Consequently, we will keep track of these attacks that have to be considered for defense, but cannot themselves be used to defend an argument.
		We will call attacks of this kind \emph{mitigated}.  
		
	\end{example}
	\begin{definition}[Mitigated Attacks]
		Let $SF=(A,R)$ be a SETAF.
		Moreover, let $E\subseteq A$ 
		and $S\in \SCCs(SF)$. 
		We define the set of \emph{mitigated attacks} as follows:
		$\Mscc{SF}{S}{E}=\{(T,h)\in R(\projIII{SF}{(E\setminus S)^+}{\UPscc{SF}{S}{E}})\mid
		\forall (T',h)\in R:T'\supseteq T \Rightarrow (T'\setminus T)\not\subseteq E\}$.
	\end{definition}
	For the computation of mitigated attacks only the ancestor SCCs are relevant.
	In particular, the set $(T'\setminus T)$ is contained in ancestor SCCs of $S$ for each attack $(T',h)\in R$. 
	To account for these novel scenarios we adapt the notion of acceptance. 
	We have to assure that the ``counter-attacks'' used for defense are not mitigated.
	\begin{definition}[Semantics Considering $C$, $M$]\label{def:semantics2}
		Let $SF=(A,R)$ be a SETAF, and let $E,C\subseteq A$ and $M\subseteq R$.
		An argument $a\in A$ is \emph{acceptable considering $M$ w.r.t.\ $E$} 
		if for all $(T,a)\in R$ there is $(X,t)\!\in\! R\setminus M$ s.t.\ $X\!\subseteq \! E$ and $t\in T$.
		\begin{itemize}
			\item $E$ is \emph{admissible in $C$ considering $M$}, denoted by $E\in \adm(SF,C,M)$,
			if $E\subseteq C$, $E\in \cf(SF)$, and each $a\in E$ is acceptable considering $M$ w.r.t.\ $E$.
			\item $E$ is \emph{complete in $C$ considering $M$}, denoted by $E\in \com(SF,C,M)$,
			if $E\in \adm(SF,C,M)$ and $E$ contains all $a\in C$ acceptable considering $M$ w.r.t.\ $E$.
			\item $E$ is \emph{preferred in $C$ considering $M$}, denoted by $E\in \pref(SF,C,M)$,
			if $E\in \adm(SF,C,M)$ and there is no $E'\in \adm(SF,C,M)$ with $E'\supset E$.
		\end{itemize}
		The \emph{characteristic function $F^{M}_{SF,C}$ of $SF$ in $C$ considering $M$}
		is the mapping $F^{M}_{SF,C}\!:\! 2^C\!\rightarrow\! 2^C$
		where\\
		$F^{M}_{SF,C}(E)= \{a\!\in\! C\mid a \text{ is acceptable  considering }M\text{ w.r.t.~} E \} $.
		\begin{itemize}
			\item $E$ is \emph{grounded in $C$ considering $M$}, denoted by $E\in \grd(SF,C,M)$,
			if $E$ is the least fixed point of $F^{M}_{SF,C}$.
		\end{itemize}
	\end{definition}
	Setting $C=A$ and $M=\emptyset$ recovers the original semantics,
	in these cases we will omit writing the respective parameter.
	The standard relationships hold for this generalization of the semantics, as the following result illustrates.
	\begin{theorem}\label{theorem:new_semantics}
		Let $SF$ be a SETAF, and let $C\subseteq A$ and $M\subseteq R$.
		Then,
		(1) $F^{M}_{SF,C}$ is monotonic,
		(2) $E\in \grd(SF,C,M)$ is the least set in $\com(SF,C,M)$ w.r.t.\ $\subseteq$, and
		(3) $E\in \pref(SF,C,M)$ are the maximal sets in $\com(SF,C,M)$ w.r.t.\ $\subseteq$.
	\end{theorem}
	
	We now redefine Definition~\ref{principle:scc1} in order to capture the admissibility-based semantics.
	\begin{definition}[SCC-recursiveness]\label{def:scc2}
		A semantics $\sigma$ satisfies \emph{SCC-recursiveness} if and only if for all SETAFs $SF=(A,R)$
		it holds $\sigma(SF)=\GF(SF, A, \emptyset)$,
		where $\GF(SF,C,M)\subseteq 2^A$ 
		is defined as:
		
		$E\subseteq A\in \GF(SF,C,M)$ if and only if
		\begin{itemize}
			\item if $|\SCCs(SF)|=1$, $E\in \BF(SF,C,M)$,
			\item otherwise, $\forall S\in \SCCs(SF)$ it holds\\
			$E\cap S \in \GF({\projIII{SF}{(E\setminus S)^+}{\UPscc{SF}{S}{E}}},\Uscc{SF}{S}{E}\cap C, \Mscc{SF}{S}{E} )$,
		\end{itemize}
		where $\BF$ 
		maps $SF=(A,R)$ with $|\SCCs(SF)|=1$ and sets $C\subseteq A$, $M\subseteq R$ to a subset of $2^A$.
	\end{definition}
	
	Towards an SCC-recursive characterization of admissible sets
	we discuss the following auxiliary results.
	Lemma~\ref{lemma:scc_adm1} shows that global acceptability implies local acceptability, Lemma~\ref{lemma:scc_adm2} shows the converse direction.
	
	\begin{lemma}\label{lemma:scc_adm1}
		Let $SF=(A,R)$ be a SETAF, let $E\in \adm(SF)$ be an admissible set of arguments,
		and let $a$ be acceptable w.r.t.\ $E$ in $SF$, where $a$ is in the SCC $S$.
		Then it holds $a\in \Uscc{SF}{S}{E}$ and $a$ is acceptable w.r.t.\
		$(E\cap S)$ in $\projIII{SF}{(E\setminus S)^+}{\UPscc{SF}{S}{E}}$ considering $\Mscc{SF}{S}{E}$.
		Moreover, $(E\cap S)$ is conflict-free in $\projIII{SF}{(E\setminus S)^+}{\UPscc{SF}{S}{E}}$.
		
	\end{lemma}
	\begin{proof}
		By the fundamental lemma~\citep{NielsenP06}
		we get $E\cup\{a\}\in \adm(SF)$, 
		and in particular $a\in \Uscc{SF}{S}{E}$ and $(E\cap S)\subseteq \Uscc{SF}{S}{E}$ for all $S\in \SCCs(SF)$.
		The key idea is that if (parts of) attacks towards an argument $x\in E$
		appear on the local level, (parts of) a global counter-attack also appear in $x$'s SCC.
		This applies both to $(E\cap S)$ and $a$.
		Regarding local conflict-freeness, this follows from the global admissibility
		of $E$: there would be a defending attack $(T,h)\in R$ against any attack violating local conflict-freeness with $T\cup \{h\}\subseteq E$, a contradiction.
	\end{proof}
	
	\begin{lemma}\label{lemma:scc_adm2}
		Let $SF=(A,R)$ be a SETAF, let $E\subseteq A$ 
		such that $(E\cap S)\in \adm(\projIII{SF}{(E\setminus S)^+}{\UPscc{SF}{S}{E}},\Uscc{SF}{S}{E},$ $\Mscc{SF}{S}{E})$ for all $S\in \SCCs(SF)$.
		Moreover, let $S'\!\in\! \SCCs(SF)$
		and let $a\!\in\! \Uscc{SF}{S'}{E}$
		be
		acceptable w.r.t.\ $(E\cap S')$ in 
		$\projIII{SF}{(E\setminus S')^+}{\UPscc{SF}{S'}{E}}$ considering $\Mscc{SF}{S'}{E}$.
		Then $a$ is acceptable w.r.t.\ $E$ in $SF$.
	\end{lemma}
	\begin{proof}
		We distinguish in 3 cases the relationship of an attack $(T,a)\in R$ to $S'$:
		(1)~$T\subseteq S'$. Either $T\cap \Dscc{SF}{S'}{E}\not=\emptyset$
		or there is a non-mitigated local counter-attack that corresponds to a global counter-attack.
		(2)~$T\subseteq A\setminus S'$. But then $T\cap E^+_R\not=\emptyset$.
		(3)~$T\cap S'\not=\emptyset$ and $T\cap (A\setminus S')\not=\emptyset$. Here we proceed as in (1).
		In all cases $a$ is defended.
	\end{proof}
	Combining these two results we obtain the SCC-recursive characterization of admissible sets.
	\begin{proposition}\label{prop:adm_key}
		Let $SF=(A,R)$
		be a SETAF
		and let $E\subseteq A$ be a set of arguments.
		Then $\forall C\subseteq A$ it holds $E\in \adm(SF,C)$ if and only if $\forall S\in \SCCs(SF)$ it holds
		$(E\cap S)\in  \adm(\projIII{SF}{(E\setminus S)^+}{\UPscc{SF}{S}{E}},\Uscc{SF}{S}{E}\cap C,\Mscc{SF}{S}{E})$. 
	\end{proposition}
	\begin{proof}
		From Lemma~\ref{lemma:scc_adm1} we get the ``$\Rightarrow$'' direction.
		For the ``$\Leftarrow$'' direction we first establish global conflict-freeness.
		Assume towards contradiction there is an attack $(T,h)\in R$ with $T\cup \{h\}\subseteq E$.
		We distinguish in 3 cases the relationship of $T$ to $h$'s SCC $S$:
		(1)~$T\subseteq S$ contradicts local conflict-freeness.
		(2)~$T\subseteq A\setminus S$ contradicts our assumption $h\in \Uscc{SF}{S}{E}$.
		It has to be $E\att_R T$, as otherwise part of the attack appears locally,
		violating local conflict-freeness.
		This attack $(X,t)\in E$ with $X\subseteq E$, where again we can handle the relationship of $X$ to $t$'s SCC $S$.
		Again, case (1) and (2) lead to contradictions,
		and due to finiteness we can only apply (3) until we encounter an initial SCC.
		Inductively, we get back to the initial case, a contradiction.
		Hence, global conflict-freeness holds.
		Global defense follows from Lemma~\ref{lemma:scc_adm2}.
	\end{proof}
	The base function for admissible sets is $\adm(SF,C,M)$.
	We will utilize this result to obtain the characterizations of the other (admissibility-based) semantics.
	\begin{theorem}
		Admissible semantics is SCC-recursive.
	\end{theorem}
	
	\subsection{Further Semantics}
	Utilizing the characterization of admissible sets,
	we can formalize the SCC-recursive scheme grounded, complete, and preferred extensions in a similar manner.
	The key idea is that local admissibility implies global admissibility and vice versa.
	The respective minimality/maximality restrictions also carry over in both directions.
	\begin{proposition}
		Let $SF=(A,R)$
		be a SETAF, let $E\subseteq A$, 
		and let $\sigma\in \{\grd,\com,\pref\}$.
		Then $\forall C\subseteq A$ it holds $E\in \sigma(SF,C)$ if and only if $\forall S\in \SCCs(SF)$ it holds
		$(E\cap S)\in  \sigma(\projIII{SF}{(E\setminus S)^+}{\UPscc{SF}{S}{E}},\Uscc{SF}{S}{E}\cap C,\Mscc{SF}{S}{E})$. 
	\end{proposition}
	As it is the case with AFs, the respective base functions 
	$\grd(SF,C,M)$, $\com(SF,C,M)$, and $\pref(SF,C,M)$ can be obtained in the same way as for admissible semantics.
	\begin{theorem}
		Grounded, complete, and preferred semantics are SCC-recursive.
	\end{theorem}
	
	\noindent Concerning the base function for $\grd$ we can use that each SCC is either cyclic or contains exactly one argument and no attack. We can thus alternatively characterize the base function as $\BF(SF,C)=\{\{a\}\}$ if $A(SF)=C=\{a\}, R(SF)=\emptyset$ and 
	$\BF(SF,C)=\{\emptyset\}$ otherwise.
	
	\subsection{Connection to Directionality}
	As it is the case in AFs, we can obtain results regarding directionality using SCC-recursiveness if the base function always admits at least one extension~\citep{BaroniG07}.
	First note that for an uninfluenced set $U$ any SCC $S$ with $S\cap U\not=\emptyset$ has to be contained in $U$, as well as all ancestor SCCs of $S$.
	Then, by the SCC-recursive characterization we get the following general result, subsuming the semantics under our consideration.
	\begin{proposition}\label{prop:directionality}
		Let $\sigma$ be a semantics such that for all SETAFs $SF$ and all $C\subseteq A(SF)$, $M\subseteq R(SF)$ it holds $\BF(SF,C,M)\not=\emptyset$.
		If $\sigma$ satisfies SCC-recursiveness then it satisfies directionality.
	\end{proposition}	
	
	\section{Incremental Computation}\label{sec:incremental}
	In this section we briefly discuss the computational implications of a semantics
	satisfying directionality, modularization, or SCC-recursiveness.
	First, for a semantics $\sigma$ satisfying directionality an argument $a$ is in some extension (in all extensions) if and only if it is in some extension (in all extensions) of the framework that is restricted to the arguments that influence $a$. That is, when deciding credulous or skeptical acceptance of an argument, in a preprocessing step, we can shrink the framework to the relevant part.
	The property of modularization is closely related to CEGAR style algorithms for preferred semantics that can be implemented via iterative SAT-solving~\citep{DvorakJWW14}. In order to compute a preferred extension 
	we can iteratively compute a non-empty admissible set of the current framework,
	build the reduct w.r.t.\ this admissible set, and repeat this procedure on the reduct
	until the empty set is the only admissible set. The preferred extension is then given by the union of the admissible sets.
	
	Finally, for SCC-recursive semantics we can iteratively compute extensions along the SCCs of a given framework (see \citep{Baumann11,LiaoJK11,BaroniGL14,CeruttiGVZ14} for such approaches for AFs). 
	That is, in the initial SCCs we simply compute the extensions and then for each of them proceed on the remaining SCCs. We then iteratively proceed on SCCs in their order. To evaluate an SCC that is attacked by other ones we have to take the attacks from earlier SCCs into account and, as we have already fixed our extension there, we can simply follow the SCC-recursive schema. We next illustrate this for stable semantics.
	
	\begin{example}
		Consider the following SETAF $SF$. 
		
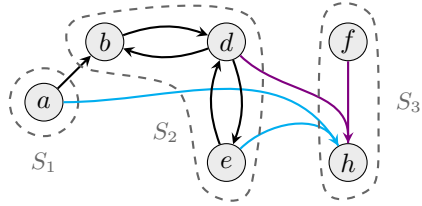
\begin{center}
			\begin{tikzpicture}[scale=0.8,>=stealth]
			\path (-1,-1) node[draw,thick,dashed,color=black!60,fill=white!0,minimum size=0.9cm,circle](s1){}
			(-1,-2) node[color=black!60] (){\small $S_1$}
			(1,-1.5) node[color=black!60] (){\small $S_2$}
			(5.,-1) node[color=black!60] (){\small $S_3$}
			;
			\path 
			(-1,-1)node[arg] (a){$a$}
			(0,0)node[arg] (b){$b$}
			(2,0)node[arg] (d){$d$}
			(2,-2)node[arg] (e){$e$}
			(4,0)node[arg] (f){$f$}
			(4,-2)node[arg] (h){$h$}
			;
			\path[thick,->]
			(a)edge(b)
			(a)edge[color=cyan,out=0,in=120,-](h)
			(e)edge[color=cyan,out=45,in=120](h)
			(d)edge[color=violet,out=-45,in=90](h)
			(f)edge[color=violet,out=-90,in=90,-](h)
			;
			\path[thick,->,bend left=20]
			(d)edge(b)
			(b)edge(d)
			(d)edge(e)
			(e)edge(d)
			;
			\draw [thick,dashed,color=black!60] plot [mark=none, smooth cycle] coordinates {(-0.4,0.4) (2.4,0.4) (2.4,-2.4) (1.6,-2.4) (1.2,-0.8) (-0.4,-0.4)};
			\draw [thick,dashed,color=black!60] plot [mark=none, smooth cycle] coordinates {(3.6,0.4) (4.4,0.4) (4.5,-1.4) (4.4,-2.4) (3.6,-2.4) };
			
			\end{tikzpicture}
		\end{center}
		We can  iteratively compute the stable extensions of SF as follows: 
		in the first SCC $S_1=\{a\}$ we simple compute all the stable extensions, i.e., $\stb(\projIII{SF}{\emptyset}{S_1})=\{\{a\}\}$.
		We then proceed with $\{a\}$ as extension $E$ for the part of the SETAF considered so far.
		Next we consider $S_2$ and adapt it to take $E$ into account.
		As $(E\setminus S_2)^+=\{b\}$ we only have to delete the argument $b$ from $S_2$ before evaluating the SCC and thus we obtain $\projIII{SF}{\{b\}}{S_2} = (\{d,e\},\{(d,e),(e,d)\})$.
		Combining these with $E$ we obtain two stable extensions $E_1=\{a,d\}$, $E_2=\{a,e\}$ for $\projIII{SF}{\emptyset}{S_1\cup S_2}$.
		We proceed with $S_3$ and first consider $E_1$.
		As $(E_1\setminus S_3)^+=\{b,e\}$ we do not remove arguments from $S_3$.
		However, as $d \in E_1$ we cannot delete the attack $(\{d,f\},h)$ but have to replace it by the attack $(f,h)$. We then have $\stb(\projIII{SF}{\{b,e\}}{S_3})=\{\{f\}\}$ and thus obtain the first stable extensions of $SF$ $\{a,d,f\}$.
		Now consider $E_2$. We have that $E_2$ attacks $h$, i.e., $(E_2\setminus S_3)^+=\{b,d,h\}$,
		and thus we have to remove $h$ before evaluating $S_3$ and thus obtain 
		$\projIII{SF}{\{b,d,h\}}{S_3}=(\{f\},\emptyset)$. 
		We end up with $\{a,e,f\}$ as the second stable extension of $SF$.        
	\end{example}
	The computational advantage of the incremental computation is that certain computations are performed over single SCCs instead of the whole framework. This is in particular significant for preferred semantics where the $\subseteq$-maximality check can be done within the SCCs. 
	Notice that verifying a preferred extension is in general $\coNP$-complete \citep{DvorakD18,DvorakGW18}. 
	However, given our results regarding the SCC-recursive scheme, the following parameterized tractability result is easy to obtain.
	\begin{theorem}
		Let $SF$ be a SETAF where $|S|\leq k$ for all $S\in \SCCs(SF)$.
		Then we can verify a given preferred extensions in $O\!\left(2^k \cdot poly(|SF|)\right)$ for some polynomial $poly$.
	\end{theorem}
	\section{Conclusion}\label{section:conclusion}
	In this work, we systematically analyzed semantics for SETAFs
	using a principles-based approach (see Table~\ref{table:principles} for an overview of the investigated
	properties).
	We pointed out interesting concepts that help us to
	understand the principles more deeply:
	edge cases that for AFs are hidden behind simple syntactic notions
	have to be considered explicitly for SETAFs,
	revealing semantic peculiarities that are already there in the special case.
	To this end, we highlight the usefulness of the \emph{reduct} in this context---many
	seemingly unrelated notions from various concepts boil down to formalizations
	closely related to the reduct.
	We particularly focused on computational properties like modularization and SCC-recursiveness.
	The latter concept has recently been investigated for Abstract Dialectical Frameworks
	in a different context by~\citeauthor{GagglRS21}~(\citeyear{GagglRS21}).
	However, it is not immediately clear how their results apply in the context of SETAFs.
	Future work hence includes the investigation of this connection.
	Moreover, the reduct for SETAFs and the generalization of the recursive scheme for SETAFs allow for the definition of new semantics that have not yet been studied in the context of collective attacks.
	An interesting direction for future works is investigating semantics \textit{cf2}~\citep{BaroniGG05a} and \textit{stage2}~\citep{DvorakG16}, as well as the family of semantics based on weak admissibility~\citep{BaumannBU2020}. 
	
	\subsubsection*{Acknowledgments}
	This research has been supported by the
	Vienna Science and Technology Fund (WWTF) through project ICT19-065,
	the Austrian Science Fund (FWF) through
	project P32830, 
	and by the German Federal Ministry of Education and Research (BMBF, 01/S18026A-F) by funding the competence center for Big Data and AI ``ScaDS.AI'' Dresden/Leipzig.
	
	\bibliographystyle{kr}
	\bibliography{references}
	
	\cleardoublepage
	\appendix
	\section{Proofs of Section~\ref{section:reduct}}
	\begin{paragraph}{Theorem~\ref{theorem:modularization}}
		Let $\F$ be a SETAF,
		$\sigma\in\{\adm,\com\}$ and $E\in\sigma(\F)$.
		\begin{enumerate}
			\item If $E'\in\sigma( \F^E )$, then $E\cup E'\in\sigma(\F)$. 
			\item If $E\cap E' = \emptyset$ and $E\cup E'\in\sigma(\F)$, then $E'\in\sigma(\F^E)$. 
		\end{enumerate}
	\end{paragraph}
	\begin{proof}
		For $\com$ semantics we utilize the results for $\adm$:
		
		1) We have $E\cup E'\in\adm(\F)$. 
		Moreover, $E'$ is complete, i.e.\ $(\F^E)^{E'}$ does not contain unattacked arguments in the reduct $\F^E$ (see Proposition~\ref{prop:classical semantics and reduct SETAFs}).
		Lemma~\ref{prop:reductbasics setafs}, item 5, implies that $\F^{E\cup E'}$ does not contain unattacked arguments, either.
		Hence $E\cup E'\in\com(\F)$. 
		
		2) Given $E\cup E'\in\com(\F)$ we have $E'\in\adm\big( \F^E \big)$, as established. 
		Regarding completeness, we again use the fact that 
		$\F^{E\cup E'} = (\F^E)^{E'}$
		does not contain unattacked arguments. 
	\end{proof}
	
	\section{Proofs of Section~\ref{section:sccs}}
	
	\begin{paragraph}{Lemma~\ref{lemma:reduct_restriction}}
		Let $SF=(A,R)$ be a SETAF and let $E,S\subseteq A$.
		Then\\
		$\projIII{SF}{(E\setminus S)^+}{S} = \projI{ SF^{(E\setminus S)}}{S}
		=(A',R')$.
	\end{paragraph}
	\begin{proof}
		We first show $A(\projIII{SF}{(E\setminus S)^+}{S})=A(\projI{ SF^{(E\setminus S)}}{S})$.
		We have
		$(A\cap S)\setminus (E\setminus S)^+ =
		(A\setminus (E\setminus S)^+)\cap S =
		(A\setminus ((E\setminus S)^+\cup (E\setminus S)))\cap S=(A\setminus (E\setminus S)^\oplus)\cap S$.
		Moreover, then it holds $R(\projIII{SF}{(E\setminus S)^+}{S})=R(\projI{ SF^{(E\setminus S)}}{S})$, as
		for some $(T,h)\in R(\projIII{SF}{(E\setminus S)^+}{S})$
		with $h\in A'$ and $T\cap (E\setminus S)^+=\emptyset$
		it holds $T\cap A'=\emptyset$ if and only if $T\not \subseteq (E\setminus S)$
		and $(T\cap A',h)$=$(T\setminus (E\setminus S),h)$.	
	\end{proof}
	
	\subsection{Stable Extensions}
	\begin{paragraph}{Lemma~\ref{lem:noPinStable}}
		
		Let $SF$ be a SETAF and
		$E\!\in\! \stb(SF)$, then
		for all $S\!\in\! \SCCs(SF)$ it holds $\Pscc{SF}{S}{E}\!=\!\emptyset$.
	\end{paragraph}
	\begin{proof}
		Assume towards contradiction that for some SCC $S$ there is an argument $a\in \Pscc{SF}{S}{E}$.
		Then, by definition there is an attack $(T,a)\in R(SF)$ such that $T\subseteq A(SF)\setminus S$ and $T\cap E^+ = \emptyset$. Moreover, $a\notin \Dscc{SF}{S}{E}$ by definition, i.e.\
		$T\not\subseteq E$. 
		But then there is some $t\in T$ such that neither $t\in E^+$ nor $t\in E$,
		which is a contradiction to the assumption that $E$ is stable.
	\end{proof}
	\begin{paragraph}{Proposition~\ref{lemma:sccStableKey}}
		Let $SF=(A,R)$ be a SETAF and let $E \subseteq A$, then
		$E\in \stb(SF)$ if and only if $\forall S\in \SCCs(SF) $ it holds
		$(E\cap S) \in \stb(\projIII{SF}{(E\setminus S)^+}{\UPscc{SF}{S}{E}})$.
	\end{paragraph}
	\begin{proof}
		Let $SF'=\projIII{SF}{(E\setminus S)^+}{\UPscc{SF}{S}{E}}$ for an arbitrary SCC $S$.
		We start by assuming $E\in \stb(SF)$.
		We need to show that $(E\cap S) \in \stb(SF')$, i.e.:
		\begin{enumerate}
			\item $(E\cap S)\subseteq \UPscc{SF}{S}{E}$,
			\item $(E\cap S)$ is conflict-free in $SF'$, and
			\item $\forall a \in \UPscc{SF}{S}{E}$ if $a\notin (E\cap S)$ then $(E\cap S)$ attacks $a$ in $SF'$.
		\end{enumerate}
		For condition 1.\ note that $(E\cap X)\cap \Dscc{SF}{X}{E}=\emptyset$ holds for any $X\subseteq A$, as otherwise $E$ would not be conflict-free in $SF$.
		For condition 2., assume towards contradiction that there is some $(T,h)\in R(SF')$ such that
		$T\cup \{h\}\subseteq (E\cap S)$.
		This means there is some $(T',h)\in R$ with $T'\supseteq T$.
		But by construction we would have $T'\setminus T \subseteq E$,
		and therefore $T'\cup \{h\}\subseteq E$, a contradiction to conflict-freeness of $E$.
		For condition 3.\ 
		we consider an arbitrary argument $a\in \UPscc{SF}{S}{E}\setminus (E\cap S)$.
		Since $a\notin E$ and $E$ is stable, there is an attack $(T,a)\in R$ with $T\subseteq E$.
		Moreover, as $a\in \UPscc{SF}{S}{E}$, it holds
		$a\notin \Dscc{SF}{S}{E}$, 
		i.e.\ in particular $T\not\subseteq (E\setminus S)$,
		or in other words $T\cap S\not= \emptyset$.
		This means by the definition of the restriction and since $T\cap E^+_R=\emptyset$ (otherwise $E$ would not be conflict-free in $SF$) there is an attack $(T\cap S, a)\in R(SF')$ with $(T\cap S)\subseteq E$.
		
		Now assume $\forall S\in \SCCs(SF) $ it holds
		$(E\cap S) \in \stb(\projIII{SF}{(E\setminus S)^+}{\UPscc{SF}{S}{E}})$.
		We need to show $E\in \stb(SF)$, i.e.\
		\begin{enumerate}
			\item $E$ is conflict-free in $SF$, and
			\item $E$ attacks all $a\in A\setminus E$ in $SF$.
		\end{enumerate}
		For 1., assume towards contradiction there is some $(T,a)\in R$ such that
		$T\cup \{a\} \subseteq E$.
		Let $S$ be the SCC containing $a$.
		Clearly $T\cup \{a\} \not\subseteq S$, as this violates our assumed conflict-freeness in $\UPscc{SF}{S}{E}$.
		Moreover, we do not have $T\subseteq (A\setminus S)$, as this would mean
		$a\in \Dscc{SF}{S}{E}$.
		Hence, there is an attack $(T\cap S,a) \in R(\projIII{SF}{(E\setminus S)^+}{\UPscc{SF}{S}{E}})$
		such that $(T\cap S)\cup \{a\}\subseteq E\cap S$, a contradiction.
		For condition 2., let us consider an arbitrary argument $a\in A\setminus E$ and let $S$ be the SCC containing $a$.
		Then either (i) $a\in  \Dscc{SF}{S}{E}$ or (ii) $a\in  \UPscc{SF}{S}{E}$.
		In case (i) we immediately get $E$ attacks $a$.
		For case (ii), we have $a\notin (E\cap S)$, and by assumption $a$ is attacked
		in $S$, i.e.\ there is an attack $(T,a)\in R(\projIII{SF}{(E\setminus S)^+}{\UPscc{SF}{S}{E}})$.
		By construction of the restriction, this means there is an attack
		$(T',a)\in R$ s.t.\ $T'\supseteq T$ and $T'\setminus T\subseteq E$.
		Hence, $T\subseteq E$, i.e.\ $E$ attacks $a$ in $SF$.
	\end{proof}
	\subsection{Admissible Sets}
	\begin{paragraph}{Theorem~\ref{theorem:new_semantics}}
		Let $SF$ be a SETAF, and let $C\subseteq A$ and $M\subseteq R$.
		Then,
		\begin{enumerate}
			\item $F^{M}_{SF,C}$ is monotonic,
			\item $E\in \grd(SF,C,M)$ is the least set in $\com(SF,C,M)$ w.r.t.\ $\subseteq$, and
			\item $E\in \pref(SF,C,M)$ are the maximal sets in $\com(SF,C,M)$ w.r.t.\ $\subseteq$.
		\end{enumerate}
	\end{paragraph}
	\begin{proof}
		(i) 
		Monotonicity of the mapping\\
		$F^{M}_{SF,C}(E)= \{a\in C\mid a \text{ is acceptable considering }M\text{ w.r.t.~} E \} $
		holds by definition of defense. 
		
		(ii)
		Setting 
		$G = \bigcup_{i\geq 1} (F^{M}_{SF,C})^i(\emptyset)$ 
		we claim that $G$ is the least set in\\
		$\com(SF,C,M)$. 
		Observe that we can apply the fundamental lemma since defense in $C$ w.r.t.\ $M$ implies the usual notion of defense. Hence admissibility of $\emptyset$ implies
		$\bigcup_{n\geq i\geq 1} (F^{M}_{SF,C})^i(\emptyset) \in \adm(SF,C,M)$ 
		for each $n$. Since $SF$ is finite and by monotonicity, 
		$\bigcup_{n\geq i\geq 1} (F^{M}_{SF,C})^i(\emptyset) = \bigcup_{i\geq 1} (F^{M}_{SF,C})^i(\emptyset) = G$
		for some $n$.  
		Thus, $G$ is complete. 
		Now let $E\in \com(SF,C,M)$. 
		By monotonicity of 
		$F^{M}_{SF,C}$ 
		we get 
		$F^{M}_{SF,C} (\emptyset) \subseteq F^{M}_{SF,C}(E)$. 
		By induction, 
		$(F^{M}_{SF,C})^i (\emptyset) \subseteq (F^{M}_{SF,C})^i(E)$
		therefore also holds for any integer $i\geq 1$. 
		Since $E$ is complete, 
		$E= (F^{M}_{SF,C})^i(E)$ 
		holds for each integer $i$, \ie the RHS is actually constant ($E$ is a fixed point of $F^{M}_{SF,C}$). 
		We conclude for each $n$
		\begin{align*}
		G =  \bigcup_{i\geq 1} (F^{M}_{SF,C})^i(\emptyset)  &=  \bigcup_{n\geq i\geq 1} (F^{M}_{SF,C})^i(\emptyset)\\
		& \subseteq  \bigcup_{n\geq i\geq 1} (F^{M}_{SF,C})^i(E)\\
		&= E,
		\end{align*}
		thus it follows that $G\subseteq E$. 
		
		(iii) 
		By definition $E\in\prf(SF,C,M)$ is maximal in $\adm(SF,C,M)$. 
		So we show $E$ is maximal in $\adm(SF,C,M)$ iff $E$ is maximal in $\com(SF,C,M)$
		
		($\Rightarrow$)
		Suppose $E\in\pref(SF,C,M)$ is not maximal in $\com(SF,C,M)$. 
		Then there is a proper complete superset $E'$ of $E$; since $E'$ is in particular admissible, $E$ is not maximal in $\adm(SF,C,M)$.
		
		($\Leftarrow$)
		Now suppose $E$ is not maximal in $\adm(SF,C,M)$.\\
		Take $E'\in\adm(SF,C,M)$ with $E\subset E'$.
		By the fundamental lemma and monotonicity of $F^{M}_{SF,C}$, we find that
		$\bigcup_{i\geq 1} (F^{M}_{SF,C})^i(E')$
		is a complete proper superset of $E$ (analogous to (ii)).
		Hence $E$ is not maximal in $\com(SF,C,M)$.
	\end{proof}
	\begin{paragraph}{Lemma~\ref{lemma:scc_adm1}}
		Let $SF=(A,R)$ be a SETAF, let $E\in \adm(SF)$ be an admissible set of arguments,
		and let $a$ be acceptable w.r.t.\ $E$ in $SF$, where $a$ is in the SCC $S$.
		Then it holds $a\in \Uscc{SF}{S}{E}$ and $a$ is acceptable w.r.t.\
		$(E\cap S)$ in $\projIII{SF}{(E\setminus S)^+}{\UPscc{SF}{S}{E}}$ considering $\Mscc{SF}{S}{E}$.
		Moreover, $(E\cap S)$ is conflict-free in $\projIII{SF}{(E\setminus S)^+}{\UPscc{SF}{S}{E}}$.
		
	\end{paragraph}
	\begin{proof}
		We will denote by $SF'$ the SETAF $\projIII{SF}{(E\setminus S)^+}{\UPscc{SF}{S}{E}}$.
		By the fundamental lemma~\citep{NielsenP06}
		we get that $E\cup\{a\}\in \adm(SF)$, i.e.\ $a\notin \Dscc{SF}{S}{E}$ by conflict-freeness and $a\notin \Pscc{SF}{S}{E}$ by defense.
		Hence, $a\in \Uscc{SF}{S}{E}$.
		Likewise, we get $(E\cap S)\subseteq \Uscc{SF}{E}{S}$, and therefore
		$(E\cap S)\subseteq A(SF')$.
		To show that $a$ is acceptable in this context we have to consider attacks $(T,a)\in R(SF')$ towards $a$ and establish
		$T\not\subseteq E$ (conflict-freeness)
		and $(E\cap A(SF'))$ attacks $T$ in $SF'$ (defense).
		As $E$ is admissible in $SF$, there is a counter-attack $(X,t)\in R$ with $t\in T$ and $X\subseteq E$.
		In particular, this means that $t\notin E$, as $(X,t)$ would otherwise contradict the conflict-freeness of $E$.
		Hence, $T\not\subseteq E$.
		Moreover, because $(T,a)\in R(SF')$, it must be that $X\cap S\not=\emptyset$, as otherwise the attack $(T,a)$ would be deleted when we consider the SCC $S$.
		We know $t\in S$ because $S$ is an SCC and there is a path from $(X\cap S)$ to $t$ and from $t$ to $a$.
		Let $X'=X\cap S$, i.e.\ there is an attack $(X',t)\in R(SF')$.
		In other words, $E\cap S$ defends $a$ in $SF'$.
		Finally, $X\subseteq E$ implies $(X',t)\notin\Mscc{SF}{S}{E}$.
		It remains to show that $(E\cap S)$ is conflict-free in $SF'$.
		Towards contradiction assume otherwise, i.e.\ there is an attack $(T,h)\in R(SF')$ with $T\cup \{h\}\subseteq (E\cap S)$.
		This means there is an attack $(T',h)\in R$ with $T'\supseteq T$,
		and as $E$ is admissible in $SF$ there is a counter-attack $(X,t)$ with $X\subseteq E$ for some $t\in T'$.
		If $t\notin S$ then $(T,h)\notin R(SF')$, a contradiction,
		therefore $t\in S$ and by assumption $t\in E$.
		But this means $X\cup\{t\}\subseteq E$, a contradiction to the conflict-freeness of $E$ in $SF$.
	\end{proof}
	
	\begin{paragraph}{Lemma~\ref{lemma:scc_adm2}}
		Let $SF=(A,R)$ be a SETAF, let $E\subseteq A$ 
		such that $(E\cap S)\in \adm(\projIII{SF}{(E\setminus S)^+}{\UPscc{SF}{S}{E}},\Uscc{SF}{S}{E},$ $\Mscc{SF}{S}{E})$ for all $S\in \SCCs(SF)$.
		Moreover, let $S'\!\in\! \SCCs(SF)$
		and let $a\!\in\! \Uscc{SF}{S'}{E}$
		be
		acceptable w.r.t.\ $(E\cap S')$ in 
		$\projIII{SF}{(E\setminus S')^+}{\UPscc{SF}{S'}{E}}$ considering $\Mscc{SF}{S'}{E}$.
		Then $a$ is acceptable w.r.t.\ $E$ in $SF$.
	\end{paragraph}
	\begin{proof}
		We have to show for each $(T,a)\in R$ 
		that $E$ attacks $T$ in $SF$.
		Let $SF'=\projIII{SF}{(E\setminus S')^+}{\UPscc{SF}{S'}{E}}$.
		We distinguish the following three cases:
		(1)~$T\subseteq S'$. If $T\cap \Dscc{SF}{S'}{E}\not=\emptyset$ we are done.
		Otherwise, all $t\in T$ are in $\UPscc{SF}{S'}{E}$.
		Then, $(T,a)\in R(SF')$ and there must be a (not mitigated) counter-attack $(X,t)$ with $t\in T$ and $X\subseteq E\cap S'$ within $SF'$,
		as we assumed $a$ is acceptable w.r.t.\ $E\cap S'$ in $SF'$ considering $\Mscc{SF}{S'}{E}$.
		This means there is an attack $(X',t)\in R$ with $X'\supseteq X$, and as $(X,t)$ is not mitigated in $SF'$ we know $X'\subseteq E$.
		In summary, $a$ is acceptable w.r.t.\ $E$ in $SF$.
		(2)~$T\subseteq A\setminus S'$. But then $T\cap E^+\not=\emptyset$ by $a\in \Uscc{SF}{S'}{E}$.
		(3)~$T\cap S'\not=\emptyset$ and $T\cap (A\setminus S')\not=\emptyset$.
		If $T\cap E^+\not=\emptyset$ we are done.
		Otherwise there is an attack $(T',a)\in R(SF')$ with $T'\subseteq T$.
		Now the reasoning proceeds as in case (1).
		As we established that there are counter-attacks in all cases (1)-(3),
		the desired property holds.%
	\end{proof}
	\begin{paragraph}{Proposition~\ref{prop:adm_key}}
		Let $SF=(A,R)$
		be a SETAF
		and let $E\subseteq A$ be a set of arguments.
		Then $\forall C\subseteq A$ it holds $E\in \adm(SF,C)$ if and only if $\forall S\in \SCCs(SF)$ it holds
		$(E\cap S)\in  \adm(\projIII{SF}{(E\setminus S)^+}{\UPscc{SF}{S}{E}},\Uscc{SF}{S}{E}\cap C,\Mscc{SF}{S}{E})$. 
	\end{paragraph}
	\begin{proof}
		Let $SF'=\projIII{SF}{(E\setminus S)^+}{\UPscc{SF}{S}{E}}$.
		We start with the ``$\Rightarrow$'' direction.
		Since $E\subseteq C$ and all $a\in E$ are acceptable w.r.t.\ $E$ in $SF$,
		we can apply Lemma~\ref{lemma:scc_adm1} and get that every $a\in (E\cap S)$ are in $\Uscc{SF}{S}{E}\cap C$ for any given SCC $S$.
		Moreover, we get that $a$ is acceptable w.r.t.\ $(E\cap S)$ in $SF'$ considering $\Mscc{SF}{S}{E}$ and that $(E\cap S)$ is conflict-free in $SF'$.
		Hence, $(E\cap S)$ is admissible in $SF'$.
		
		We continue with the ``$\Leftarrow$'' direction.
		As for all SCCs $S$ we assume $(E\cap S)\subseteq (S\cap C)$ we know $E\subseteq C$, i.e.\ we only need to show admissibility in $SF$.
		Towards contradiction assume $E$ is not conflict-free in $SF$, i.e.\ there is an attack $(T,h)\in R$ with $T\cup \{h\}\subseteq E$.
		Let $S'$ be the SCC $h$ is in.
		We cannot have (1) $T\subseteq S'$, as this would contradict the assumption that $E\cap S'$ is conflict-free in $\projIII{SF}{(E\setminus S')^+}{\UPscc{SF}{S'}{E}}$.
		Moreover, it cannot be that (2) $T\subseteq A\setminus S'$, because then $h\in \Dscc{SF}{S'}{E}$ while we assumed $h\in \Uscc{SF}{S'}{E}$.
		Finally, consider the case (3) where $T\cap S'\not=\emptyset$ and
		$T\cap (A\setminus S')\not=\emptyset$.
		Then it holds $E^+\cap T\not=\emptyset$, as otherwise there would be
		$(T\cap S',h)\in \projIII{SF}{(E\setminus S')^+}{\UPscc{SF}{S'}{E}}$,
		contradicting our assumption of local conflict-freeness.
		This means there is an attack $(X,t)\in R$ with $X\subseteq E$ and $t\in T\setminus S'$.
		Let $S''$ be the SCC $t$ is in.
		As before, we cannot have (1) $X\subseteq S''$ or (2) $X\subseteq A\setminus S''$.
		The only remaining case is again (3) $X\cap S''\not=\emptyset$ and
		$X \cap (A\setminus S'')\not=\emptyset$---as by this step (3) always takes us to a prior SCC and we assume $SF$ finite, eventually this recursion will stop in case (1) or (2).
		Now, by induction we get a contradiction for the initial case.
		
		It remains to show that every $a\in E$ is acceptable w.r.t.\ $E$ in $SF$.
		Let $S^*$ be the SCC $a$ is in and let $SF^*=\projIII{SF}{(E\setminus S^*)^+}{\UPscc{SF}{S^*}{E}}$.
		By assumption, $(E\cap S^*)\in \adm(SF^*,\Uscc{SF}{S^*}{E},\Mscc{SF}{S^*}{E})$, i.e.\ $a$ is acceptable w.r.t.\ $E\cap S$ in $SF^*$ considering $\Mscc{SF}{S^*}{E}$.
		Since we also have $a\in \Uscc{SF}{S^*}{E}$, we can apply Lemma~\ref{lemma:scc_adm2} and get that $a$ is acceptable w.r.t.\ $E$ in $SF$.%
	\end{proof}
	
	\subsection{Complete Semantics}
	We already have the tools to characterize complete extensions.
	\begin{proposition}\label{prop:com_key}
		Let $SF=(A,R)$
		be a SETAF
		and let $E\subseteq A$ be a set of arguments.
		Then $\forall C\subseteq A$ it holds $E\in \com(SF,C)$ if and only if $\forall S\in \SCCs(SF)$ it holds
		$(E\cap S)\in  \com(\projIII{SF}{(E\setminus S)^+}{\UPscc{SF}{S}{E}},\Uscc{SF}{S}{E}\cap C,\Mscc{SF}{S}{E})$. 
	\end{proposition}
	\begin{proof}
		We start with the ``$\Rightarrow$'' direction.
		From $E\in \com(SF,C)$ by Proposition~\ref{prop:adm_key} we get
		$\forall S\in \SCCs(SF)$ it holds\\
		$(E\cap S)\in \adm(\pte{S},\UPscc{SF}{S}{E}\cap C,\Mscc{SF}{S}{E})$.
		For an arbitrary SCC $S'\in \SCCs(SF)$,
		let $a\in \Uscc{SF}{S'}{E}$ be an argument acceptable w.r.t.\ $(E\cap S')$
		in $\pte{S'}$
		considering
		$\Mscc{SF}{S'}{E}$.
		By Lemma~\ref{lemma:scc_adm2}, $a$ is acceptable w.r.t.\ $E$ in $SF$, and, hence, $a\in E$ and $a\in (E\cap S')$ by completeness.
		For the ``$\Leftarrow$'' direction
		we get $E\in \adm(SF,C)$ by Proposition~\ref{prop:adm_key}.
		For an arbitrary $a\in C$, let $S'$ be the SCC $a$ is in.
		If $a$ is acceptable w.r.t.\ $E$ in $SF$,
		by Lemma~\ref{lemma:scc_adm1} we get that $a$ is
		acceptable w.r.t.\ $(E\cap S')$ in $\pte{S'}$ considering $\Mscc{SF}{S'}{E}$.
		As $(E\cap S')$ is locally complete, we get $a\in E$.%
	\end{proof}
	From this we get the desired result regarding complete extensions.
	The base function is $\com(SF,C,M)$.
	\begin{theorem}
		Complete semantics is SCC-recursive.	
	\end{theorem}
	
	\subsection{Preferred Semantics}
	The next lemma illustrates that if we already found a globally admissible set $E$ and find a (larger) locally admissible set $E'\supset E\cap S$ in an SCC $S$,
	then we can find a globally admissible set incorporating this set $E'$.
	This idea underlies the incremental computation of extensions (see Section~\ref{sec:incremental}).
	\begin{lemma}\label{lem:scc_pref_combine}
		Let $SF=(A,R)$ and let $E\in \adm(SF)$, let $S\in \SCCs(SF)$ be an SCC.
		Moreover, let $E'\subseteq A$ be a set of arguments such that
		$(E\cap S)\subseteq E'\subseteq \Uscc{SF}{S}{E}$,
		and $E'\in \adm(\projIII{SF}{(E\setminus S)^+}{\UPscc{SF}{S}{E}},\Uscc{SF}{S}{E}, \Mscc{SF}{S}{E})$.
		Then $(E\cup E')$ is admissible in $SF$. 
	\end{lemma}
	\begin{proof}
		We first show that $(E\cup E')$ is conflict-free in $SF$.
		Assume towards contradiction there is $(T,h)\in R$ with $T\cup \{h\}\subseteq (E\cup E')$.
		Either (1) $h\in E$ or (2) $h\in E'\setminus E$.
		Since $E\in \adm(SF)$ in (1) we have $E\att_R T$.
		$E\in \cf(SF)$, this means $E\att_R T'$ where $T'=T\setminus E = T\cap E'\not=\emptyset$.
		But this means $E'\cap \Dscc{SF}{S}{E}\not=\emptyset$,
		contradicting our assumption $E'\subseteq \Uscc{SF}{S}{E}$.
		Regarding (2), if $T\subseteq E$, then $h\in \Dscc{SF}{S}{E}$, a contradiction.
		Hence, $T\cap (E'\setminus E)\not=\emptyset$.
		It follows there is $(T',h)\in R(SF')$ with $T'\subseteq T$.
		However, since we assume $E'$ is conflict-free in $SF'$ it holds $T'\cup \{h\}\not\subseteq E'$, a contradiction because $(E\cap S)\subseteq E'$.
		As both possibilities lead to contradictions, we conclude $(E\cup E')\in \cf(SF)$. 
		It remains to show defense in $SF$,
		i.e.\ for all $(T,h)\in R$ with $h\in E\cup E'$,
		$E\cup E'\att_R T$.
		If $h \in E$, this follows from $E\in \adm(SF)$.
		For $h\in E'\setminus E$, we distinguish 4 cases:
		(1)~$T\cap E^+\not=\emptyset$, then we are done.
		(2)~$T\subseteq A\setminus S$. But then either $E\att_R T$---see (1)---or $E\not\att_R T$ and therefore $h\in \Pscc{SF}{S}{E}$, a contradiction to $h\in E'\subseteq \Uscc{SF}{S}{E}$.
		(3)~$T\subseteq A(SF')$. But this means $(T,h)\in R(SF')$ and therefore
		since $E'$ is admissible in this context there is $(X,t)\in R(SF')$ with $X\subseteq E'$ and $t\in T$.
		As $(X,t)$ cannot be mitigated, there is an attack $(X',t)\in R$ with $X'\supseteq X$ and $X'\setminus X\subseteq E$, i.e.\ $X'\subseteq (E\cup E')$, contradicting the earlier established conflict-freeness.
		(4)~$T\cap A(SF')\not=\emptyset$ and $T\cap (A\setminus A(SF'))\not=\emptyset$.
		If we assume $E\not\att_R T$---see (1)---there is an attack $(T',h)\in R(SF')$ that is not mitigated, and we proceed as in (3).
		In any case, there is a defense against the attack $(T,h)$, therefore,
		$(E\cup E')\in \adm(SF)$.%
	\end{proof}
	\begin{proposition}\label{prop:pref_key}
		Let $SF=(A,R)$ be a SETAF and let $E\subseteq A$ be a set of arguments.
		Then $\forall C\subseteq A$ it holds $E\in \pref(SF,C)$ if and only if $\forall S\in \SCCs(SF)$ it holds $(E\cap S)\in\pref(\projIII{SF}{(E\setminus S)^+}{\UPscc{SF}{S}{E}},C\cap \Uscc{SF}{S}{E},\Mscc{SF}{S}{E})$.		
	\end{proposition}
	\begin{proof}
		We start with the ``$\Rightarrow$'' direction.
		We assume $E\in \pref(SF,C)$, and can apply Proposition~\ref{prop:adm_key}
		and obtain that $\forall S\in \SCCs(SF)$ we have $(E\cap S)\in \adm(\pte{S},C\cap \Uscc{SF}{S}{E},\Mscc{SF}{S}{E})$.
		Assume towards contradiction that for some $S'\in \SCCs(SF)$ there is a set $E'\in \adm(\pte{S'},C\cap \Uscc{SF}{S'}{E},\Mscc{SF}{S'}{E})$ with $E\cap S'\subset E'$. 
		However, by Lemma~\ref{lem:scc_pref_combine}
		this means the set $E\cup E'$
		is in $\adm(SF,C)$, but since $E\subset E\cup E'$ this contradicts our assumption $E\in \pref(SF,C)$.
		
		For the ``$\Leftarrow$'' direction, from the assumption
		and Proposition~\ref{prop:adm_key}
		we get
		$E\in \adm(SF,C)$.
		Towards contradiction assume there is an $E'\in \adm(SF,C)$ with $E'\supset E$.
		This means there is some SCC $S\in \SCCs(SF)$
		such that $(E\cap S)\subset (E'\cap S)$.
		W.l.o.g.\ we choose $S$ such that no ancestor SCC of $S$ has this property.
		This means that $\Uscc{SF}{S}{E}=\Uscc{SF}{S}{E'}$ and $\Pscc{SF}{S}{E}=\Pscc{SF}{S}{E'}$
		for $S$ and all of its ancestor SCCs.
		Consequently, $(E'\cap S)\subseteq \Uscc{SF}{S}{E}$, and by another application of Proposition~\ref{prop:adm_key}
		we get
		$E'\in \adm(\pte{S},C\cap \Uscc{SF}{S}{E},\Mscc{SF}{S}{E})$.
		However, this contradicts our assumption\\
		$E\in \pref(\pte{S},C\cap \Uscc{SF}{S}{E},\Mscc{SF}{S}{E})$.
	\end{proof}
	From this we get the desired result regarding preferred extensions.
	The base function is $\pref(SF,C,M)$.
	\begin{theorem}
		Preferred semantics is SCC-recursive.
	\end{theorem}
	\subsection{Grounded Semantics}
	For the characterization of grounded semantics we exploit the fact that also in our setting the grounded is the unique minimal complete extension (see Theorem~\ref{theorem:new_semantics}).
	\begin{proposition}\label{prop:grd_key}
		Let $SF=(A,R)$ be a SETAF and let $E\subseteq A$ be a set of arguments.
		Then $\forall C\subseteq A$ it holds $E\in \grd(SF,C)$ if and only if $\forall S\in \SCCs(SF)$ it holds $(E\cap S)\in\grd(\projIII{SF}{(E\setminus S)^+}{\UPscc{SF}{S}{E}},C\cap\Uscc{SF}{S}{E},\Mscc{SF}{S}{E})$.
	\end{proposition}
	\begin{proof}
		We start with the ``$\Rightarrow$'' direction.
		We assume $E\in \grd(SF,C)$, and can apply Proposition~\ref{prop:com_key}
		and obtain that $\forall S\in \SCCs(SF)$ we have $(E\cap S)\in \com(\pte{S},\Uscc{SF}{S}{E}\cap C,\Mscc{SF}{S}{E})$.
		Assume towards contradiction that for some SCC $S'$ $(E\cap S')$ is not minimal among the locally complete extensions.
		W.l.o.g.\ we choose $S'$ such that no ancestor SCC of $S'$ has this property.
		Let $E'\in \grd(\pte{S},\Uscc{SF}{S}{E}\cap C,\Mscc{SF}{S}{E})$.
		We can construct $E''$ such that for the ancestor SCCs of $S'$ the new set
		$E''$ coincides with $E$, for $S'$ it coincides with $E'$, and for the remaining SCCs $S$ is determined by $\grd(\projIII{SF}{(E'\setminus S)^+}{\Uscc{SF}{S}{E'}},\UPscc{SF}{S}{E'}\cap C,\Mscc{SF}{S}{E'})$ (see Section~\ref{sec:incremental} for details).
		But then $E''\in \com(SF,C)$ by Proposition~\ref{prop:com_key}
		and $E\not\subset E''$, a contradiction to our assumption $E\in \grd(SF,C)$.
		For the $\Leftarrow$ direction we get $E\in \com(SF,C)$ by Proposition~\ref{prop:com_key}.
		Towards contradiction assume there is some $E'\subset E$ with
		$E'\in \grd(SF,C)$.
		This means there is an SCC $S$ where $(E'\cap S)\subset (E\cap S)$.
		W.l.o.g., we choose $S$ such that no ancestor SCC of $S$ has this property.
		This means that $\Uscc{SF}{S}{E}=\Uscc{SF}{S}{E'}$ and $\Pscc{SF}{S}{E}=\Pscc{SF}{S}{E'}$
		for $S$ and its ancestor SCCs.
		Consequently, $(E'\cap S)\in \com(\pte{S},\UPscc{SF}{S}{E}\cap C,\Mscc{SF}{S}{E})$.
		However, this contradicts our assumption $E\in \grd(\pte{S},\UPscc{SF}{S}{E}\cap C,\Mscc{SF}{S}{E})$, since $(E'\cap S)\subset (E\cap S)$.
	\end{proof}
	\begin{theorem}
		Grounded semantics is SCC-recursive.
	\end{theorem}
	
	\subsection{Directionality}
	Following~\citep{BaroniG07}:
	\begin{paragraph}{Proposition~\ref{prop:directionality}}
		Let $\sigma$ be a semantics such that for all SETAFs $SF$ and all $C\subseteq A(SF)$, $M\subseteq R(SF)$ it holds $\BF(SF,C,M)\not=\emptyset$.
		If $\sigma$ satisfies SCC-recursiveness then it satisfies directionality.
	\end{paragraph}
	\begin{proof}
		We use the fact that for an uninfluenced set $U$ any SCC $S$ with $S\cap U\not=\emptyset$ has to be contained in $U$, as well as all ancestor SCCs of $S$.
		Let $\mathcal{S}$ be the set of SCCs $S$ with $S\subseteq U$. 
		Considering the SCC-recursive characterization, this yields
		$\sigma(\projIII{SF}{\emptyset}{U})= \{E\subseteq U \mid \forall S\in \mathcal{S}:(E\cap S)\in \GF(\projIII{SF}{\emptyset}{\UPscc{SF}{S}{E}},\!\Uscc{SF}{S}{E},\allowbreak\Mscc{SF}{S}{E}) \}$.
		We have to show that $\sigma(\projIII{SF}{\emptyset}{U})=\{E\cap U \mid E\in \sigma(SF)\}$.
		
		We get the ``$\subseteq$'' direction from
		the fact that $\Uscc{SF}{S}{E}=\Uscc{SF}{S}{E\cap U}$
		and $\Pscc{SF}{S}{E}=\Pscc{SF}{S}{E\cap U}$
		for all $S\in \mathcal{S}$.
		The ``$\supseteq$'' direction is immediate: as we assume that $\BF(SF,C,M)$ always yields at least one extension, we can extend any set $(E\cap U)$ according to the SCC-recursive scheme (see Section~\ref{sec:incremental} for details).
	\end{proof}
	
\end{document}